\documentclass{amsart}
\usepackage{color,amsmath,graphicx,amsfonts,amssymb}

\begin{document}

\newtheorem{theorem}{Theorem}[section]
\newtheorem{proposition}[theorem]{Proposition}
\newtheorem{alg}[theorem]{Algorithm}
\theoremstyle{definition}
\newtheorem{definition}[theorem]{Definition}
\newtheorem{example}[theorem]{Example}
\theoremstyle{remark}
\newtheorem{remark}[theorem]{Remark}
\newcommand{\scst}{\scriptscriptstyle}
\newcommand{\DP}{\widetilde{\Delta}_{\scst P}}
\newcommand{\la}{\langle}
\newcommand{\ra}{\rangle}
\newcommand{\ot}{\otimes}

\title{Cups Products in $\mathbb{Z}_2$-Cohomology of 3D Poyhedral Complexes}

\thanks{This research was funded in part by the Spanish Ministry of Economy and Competitiveness under project MTM2012-32706 and a Millersville University faculty research grant.}

\author{Rocio Gonzalez-Diaz}
\address{Dept. of Applied Math (I), School of Computer Engineering, University of Seville, Campus Reina Mercedes, C.P. 41012, Seville, Spain}
\email{rogodi@us.es}

\author{Javier Lamar}
\address{Pattern Recognition Department, Advanced Technologies Application Center,
7th Avenue \#21812 218 and 222, Siboney, Playa, C.P. 12200,
Havana City, Cuba}
\email{jlamar@cenatav.co.cu}

\author{Ronald Umble}
\address{Department of Mathematics, Millersville University of Pennsylvania,
P.O. Box 1002
Millersville, PA 17551-0302, Pennsylvania, USA}
\email{ron.umble@millesville.edu}


\begin{abstract}
 Let $I=(\mathbb{Z}^3,26,6,B)$ be a 3D digital image, let $Q(I)$ be the associated cubical complex and let $\partial Q(I)$ be the subcomplex of $Q(I)$ whose maximal cells are the quadrangles of $Q(I)$ shared by a voxel of $B$ in the foreground -- the object under study -- and by a voxel of $\mathbb{Z}^3\smallsetminus B$ in the background -- the ambient space.  We show how to simplify the combinatorial structure of $\partial Q(I)$ and obtain a 
3D polyhedral complex $P(I)$ homeomorphic to $\partial Q(I)$ but with fewer cells.  We introduce an algorithm that computes cup products on $H^*(P(I);\mathbb{Z}_2)$ directly from the combinatorics. The computational method introduced here can be effectively applied to any polyhedral complex
embedded in $\mathbb{R}^3$.
\end{abstract}

\date{July 9, 2013}

\maketitle

\section{Introduction}\label{introduction}

This paper completes the work proposed in our prequel-abstract \cite{iwcia2011} by providing additional insights, examples, proofs and a corrected version of formula (\ref{formula}) in Theorem \ref{main}.  \emph{Unless explicitly indicated otherwise, all modules in the discussion that follows are assumed to have $\mathbb{Z}_2$ coefficients.}

Roughly speaking, a
\emph{regular cell complex} is a collection of $q$-dimensional cells
glued together in such a way that non-empty intersections of cells is a cell \cite[page 348]{at}.
A regular cell complex is a \emph{polyhedral complex} if its $q$-cells are $q$-polytopes \cite[pages 24--25]{kozlov}. In particular, a cubical complex is a polyhedral complex whose $q$-cells are $q$-cubes. In this paper we are especially interested in \emph{3D polyhedral complexes}, which are polyhedral complexes embedded in $\mathbb{R}^3$. 

The \emph{cellular chain complex} of a regular cell complex $X$ is the chain complex $\left(C_{\ast}(X),\partial \right)$, where $C_{\ast}(X)$ is the graded $\mathbb{Z}_2$-vector space generated by the cells of $X$ and $\partial$ is the cellular boundary; the \emph{cellular cochain complex} $(C^*(X),\delta)$ is its linear dual. The {\em cellular cohomology} $H^*(X)$ is the quotient $\operatorname{Ker}\delta / \operatorname{Im}\delta$, whose coset representatives are cocycles supported on the connected components, non-contractible loops, and surfaces bounding the cavities of $X$. Moreover, $H^*(X)$ is a graded commutative algebra whose cup product encodes certain relationships among the generators and enhances our ability to distinguish between objects.  For example, $H^*(S^1\vee S^1\vee S^2)$ and $H^*(S^1 \times S^1)$ are isomorphic as vector spaces but not as algebras since cup products of 1-dimensional classes vanish in $H^*(S^1\vee S^1\vee S^2)$ but not in $H^*(S^1 \times S^1)$. Indeed,  $S^1\vee S^1\vee S^2$ and $S^1 \times S^1$ are not homeomorphic and have quite different topological properties.  

In this paper, we introduce new formulas for computing cup products directly from the combinatorics of a 3D polyhedral complex.
When $X$ is a regular cell complex homeomorphic to but with fewer cells than the boundary surface of a 3D digital image, our formulas
improve computational efficiency in certain problems related to 3D image processing. 

To date, the cup product has seen limited application in 
3D image processing.
In \cite{GR03,GR05}, R. Gonzalez-Diaz and  P. Real used their $14$-adjacency algorithm and the standard formulation in \cite{Mun84} to compute cup products on the simplicial complex associated with a given 3D digital image. At about the same time, T. Kaczynski, M. Mischaikow and M. Mrozek \cite{KMM04} showed how to
compute the homology of a 3D digital image directly from the voxels thought of as a cubical complex.  Their algorithm, which applies techniques from linear algebra, demonstrates that the homology 
of a cubical complex 
is
actually computable. 
R. Gonzalez-Diaz, Jimenez and B. Medrano \cite{ijist} introduced a method for computing cup products directly from a cubical complex associated with a given 3D digital image (no additional subdivisions are necessary). 
More recently, motivated by problems in high-dimensional data analysis, T. Kaczynski and M. Mrozek \cite{K} gave an algorithm for computing the cohomology 
algebra of a cubical complex of arbitrary dimension; they are currently in the process of implementing their algorithm.
In the context of persistent cohomology,  A. Yarmola \cite{yarmola} discussed the computation of cup products in 
the cohomology of a finite simplicial complex over a field.
For a geometrical interpretation of cohomology in the context of digital images, we refer the reader to \cite{cviu}.

The cup product arises naturally from a \emph{diagonal approximation} on a regular cell complex $X$, which is a map $\Delta : X \to X \times X$ that preserves cellular structure and is homotopic to the geometric diagonal map 
$\Delta^G(\sigma)=\sigma \times \sigma$.
 In \cite{alumno}, D. Kravatz constructed a diagonal approximation on a general polygon and used it to compute cup products on the cohomology of closed compact orientable surfaces thought of as identification spaces of 
$2n$-gons.
Unfortunately Kravatz's diagonal depends upon a particular indexing of the vertices and cannot be applied in more general settings. Consequently, we introduce a generalization 
of Kravatz's diagonal, which is independent of indexing (Theroem \ref{main}).

A problem that frequently arises in 3D image processing is to encode the boundary surface of a 3D digital image as a set of voxels.
The most popular approach to this problem uses a triangulation.  While triangles are combinatorially simple, and visualization of triangulated surfaces is supported by existing software, the number of triangles required is often large and the computational analysis is correspondingly slow. It is desirable, therefore, to consider approximations with simpler combinatorics that can be analyzed more efficiently.  Indeed, given any covering of a 
$2$-dimensional
surface with polygons, we show how to iteratively merge adjacent polygons and 
reduce
the number of polygons in the covering.  
Having done so, we apply our generalized diagonal approximation to compute cup products directly from the combinatorics.  
Although our strategy for merging adjacent cells  is not new (see \cite{Argawal,Bro,Das,KMS}, for example), the notion that cup products can be computed using a diagonal approximation that varies from cell-to-cell is novel indeed.
For testing purposes, a Matlab implementation of our method is posted at the URL below\footnote{http://grupo.us.es/cimagroup/imagesequence2cupproduct.zip}.

More precisely, let $\mathbb{Z}^3$ denote the set of positive integer lattice points in $\mathbb{R}^3$ and consider
a 3D digital image $I=(\mathbb{Z}^3,26,6,B)$, where 
$B\subset \mathbb{Z}^3$ is the {\itshape foreground}, $B^c=\mathbb{Z}^3 \smallsetminus B$ is the {\itshape background}, 
$26$ is the adjacency relation for the foreground and $6$ for the background {(see \cite{kong}).
Represent $I$
as the set of unit cubes (voxels) centered at the points of $B$ together with all of their faces, and let $Q(I)$ be the associated cubical complex.  Let $\partial Q(I)$ be the subcomplex of $Q(I)$ whose maximal cells are the quadrangles of $Q(I)$ shared by a voxel of $B$ and a voxel of $\mathbb{Z}^3\smallsetminus B$.
We show how to merge adjacent quadrangles into a polygonal face to produce a polyhedral complex homeomorphic to $\partial Q(I)$ with fewer cells.  Our procedure differs from existing simplification procedures whose approximating complexes either fail to be cell complexes (see \cite{KMM04}, for example), are simpler triangulations of given triangulations, or are simplifications of triangulated manifolds.  
A complete list of surface simplifications appear in \cite{survey}.

The paper is organized as follows:  Section \ref{AT} reviews some standard definitions from algebraic topology. Section \ref{section_cup} reviews the notions of diagonal approximation and cup product, defines the diagonal approximation induced from a given diagonal approximation via a chain contraction, and derives some of its important properties.  Our main result, which appears in Section \ref{section_polyhedral} as Theorem \ref{main}, gives an explicit formula for a diagonal approximation on each polygon in a 
3D polyhedral complex $X$.
Since $H^3(X)\equiv 0$, the cup product defined in terms of our combinatorial diagonal is sufficient for computing the cohomology algebra $H^{\ast}(X)$.  In Section \ref{nueva} we introduce our reduction algorithm, which reduces a cubical $Q(I)$ to a homeomorphic polyhedral complex $P(I)$ with fewer cells. 
Using the explicit formula given in Theorem \ref{main}, we give an algorithm for computing cup products on $H^*(P(I))$ directly from the given combinatorics and analyze its computational complexity. Conclusions and some ideas for future work are presented in Section \ref{tres}.

\section{Standard Definitions from Algebraic Topology}\label{AT}

Unless explicitly indicated otherwise, all modules in the discussion that follows are assumed to have $\mathbb{Z}_2$ coefficients.
Let $S=\{S_q\}$ be a graded set. 
A $q$-{\em chain} is a 
finite formal sum 
 of elements of $S_q$ 
(mod $2$ addition).
The $q$-chains together with the operation of vector addition form the vector space $C_q(S)$ of $q$-chains of $S$.
The   \emph{chains of} $S$ is the graded vector space $C_*(S)=\{C_q(S)\}$, and 
a \emph{chain complex of} $S$
is a pair $(C_*(S),\partial)$, where
$\partial=\{\partial_{q}:C_{q}(S)\rightarrow C_{q-1}(S)\}$ is a square zero linear map called the \emph{boundary operator}.

A $q$-chain $a\in C_q(S)$ is a
$q$-\emph{cycle} if $\partial_q  (a)=0$; it is a $q$-\emph{boundary}
if there is a $(q+1)$-chain $a''$ such that $\partial_{q+1} (a'')=a$.
Two $q$-cycles $a$ and $a'$ are \emph{homologous} if $a+a'$ is a $q$-boundary.
Let $Z_q(S)$ and $B_q(S)$ denote the $q$-cycles and $q$-boundaries of $S$, respectively.
Then $B_q(S)\subseteq Z_{q}(S)$ since $\partial\partial=0$. The quotient $H_q(S)=Z_q(S)/B_q(S)$
is the $q^{th}$  \emph{homology 
of} $S$ and the graded 
vector space 
$H_*(S) = \{H_q(S)\}$
is the \emph{homology} of $S$.  An element of $H_q(S)$ is a class $[a] := a + B_q(S)$; the cycle $a$ is a {\em representative cycle}
of the class $[a]$. The dimension of $H_q(S)$ is called the $q^{th}$ Betti number and is denoted by $b_q$.

Let $(C_*(S),\partial)$ and $(C_*(S'),\partial')$ be chain
complexes. A linear map of graded vector spaces 
$f=\{f_{q}:C_{q}(S)\rightarrow C_{q}(S^{\prime
})\}$ (also denoted by $f:C_*(S)\to C_*(S)$) is a \emph{chain map} if
$\mbox{$f_{q-1}\partial_{q}=\partial_{q}^{\prime}f_{q}$ for all $q.$}$
Let $f,g:C_*\left(  S\right)  \rightarrow C_*\left(  S^{\prime}\right)$
be chain maps. A \emph{chain homotopy from }$f$\emph{ to} $g$ is a linear map of graded 
vector spaces
$\phi=\{\phi_{q}:C_{q}(S)\rightarrow C_{q+1}(S^{\prime})\}$ (also denoted by $\phi:C_*(S)\to C_{*+1}(S)$) such that
$\partial_{q+1}^{\prime}\phi_{q}+\phi_{q-1}\partial_{q}=f_{q}+g_{q}$ for all $q$, or simply $\partial^{\prime}\phi+\phi\partial=f+g$.

A \emph{chain contraction of} $(C_*(S),\partial)$ \emph{to} $(C_*(S^{\prime}),\partial^{\prime
})$ is a tuple $(f, g,\phi,(S,\partial),(S',\partial'))$ consisting of
chain maps $f:C_*(S)\rightarrow C_*(S')$, $g:C_*(S^{\prime})\rightarrow C_*(S)$, and 
a chain homotopy
$\phi:C_*(S)\to C_{*+1}(S)$ such that $fg=\mathbf{1}_{\scst C_*\left(  S^{\prime}\right)}$ and $\partial^{\prime}\phi+\phi\partial=\mathbf{1}_{\scst C_*\left(  S\right)}+gf$. 
Note that $f$ and $g$ are chain homotopy equivalences by definition.
The idea of a chain contraction is due to H. Cartan \cite{Cartan}, whose ``little constructions'' are contractions of the bar construction onto a tensor product of exterior algebras, polynomial algebras, and truncated polynomial algebras.  

\begin{remark}\label{property-contraction}
If $(f, g,\phi,(S,\partial),(S',\partial'))$ is a chain contraction, the chain homotopy equivalence $g$ induces and isomorphism $H_{\ast}(S^{\prime})\approx H_{\ast}(S)$.
\end{remark}

Let $n$ be a positive integer. A topological space $X$ is an  {\em $n$-dimensional regular cell complex} if $X$ can be constructed inductively in the following way (see \cite[Page 243]{Massey}):
\begin{enumerate}
\item [$\bullet$] Choose a finite discrete set $X^{(0)}$, called the $0$-{\em skeleton of} $X$; the elements  of  $X^{(0)}$ are called $0$-{\em cells} of $X$.
\item [$\bullet$] If the 
$(q-1)$-skeleton 
$X^{(q-1)}$ 
has already been constructed, construct the $q$-{\em skeleton} 
$X^{(q)}$ by attaching finitely many 
closed $q$-disks 
$\{\sigma_{i}^q\}$ 
to 
$X^{(q-1)}$
via homeomorphisms  
$\{f_{i}: \partial \sigma_{i}^q \to E_{i}^{q-1}\}$,
 where 
$E_{i}^{(q-1)}$
 is a subcomplex of 
$X^{(q-1)}$
homeomorphic to the 
$(q-1)$-dimensional
sphere. Once attached, 
$\sigma_{i}^q$ is called a 
$q$-{\em cell} of $X$.
\item [$\bullet$] The induction terminates at the $n^{th}$ stage with $X=X^{(n)}$.
\end{enumerate}
A $q$-cell $\sigma'$ of $X$ is a {\em facet} of a $(q+1)$-cell $\sigma$
of $X$ if  $\sigma'\subset \sigma$.  A {\em maximal} cell of $X$ is not a facet of any cell of $X$. 
Regular cell complexes have particularly nice properties, for example, their homology 
is
effectively computable \cite[Page 243]{Massey}.
When a regular cell complex $X$ is constructed by attaching 
polytopes, we refer to $X$ as a {\em polyhedral complex}. 
In particular, a
\emph{3D polyhedral complex} is a polyhedral complex embedded in $\mathbb{R}^3$. 

Let $X$ and $Y$ be regular cell complexes.
A continuous map $f:X \to Y$ of regular cell complexes is {\em cellular} if
$f(X^{(q)})\subseteq Y^{(q)}$ for each $q$. 
Thus if $\sigma$ is a $q$-cell of $X$, then $f(\sigma)$ is a union of cells $\{e^{q_i}\} \subset Y$ with $q_i\leq q$ for all $i$. A cellular map $f:X \to Y$ induces a map $f:C_*(X)\to C_*(Y)$ in the obvious way.

Let ${\mathcal X}=\{{\mathcal X}_q\}$ denote the graded set of cells of a regular cell complex $X$.
The \emph{cellular chains of} $X$, denoted by $C_{\ast}\left(  X\right) $, is the 
graded vector space generated by the elements of  ${\mathcal X}$.
The \emph{cellular chain complex of} $X$ is the chain complex $\left(  C_{\ast}\left(  X\right),\partial\right)  ,$ where $\partial$ is the linear extension of the cellular
boundary. The {\em cellular homology of} $X$, denoted by $H_*(X)$, is the homology of the graded set ${\mathcal X}$.

\begin{definition}\label{merging1}
Let  $X$ be a regular cell complex and let $\gamma$ be an $r$-cell of $X$ contained in the boundary of exactly two $(r+1)$-cells $\mu$ and $\mu'$ of $X$. Let $X'$ be the 
cell complex obtained from $X$ by replacing $\mu$ and $\mu'$ with the single cell $\mu''= \mu \cup \mu'$ so that 
$\mathcal{X'}=(\mathcal{X}\smallsetminus\{\gamma,\mu,\mu'\})\cup\{\mu''\} $
 and $\partial'_q: C_q(X')\to C_{q-1}(X')$ 
is the linear extension of the map defined on generators $\sigma\in \mathcal{X'}$ by
\[
\partial_{q}^{\prime}\left(  \sigma\right)  =\left\{
\begin{array}
[c]{ll}%
\partial_{r+1}\left( \mu+\mu'\right)  , &
\mbox{$q=r+1$ and $\sigma=\mu''$,}\\
\partial_{r+2}(\sigma)+\mu+\mu'+\mu'', &
\mbox{$q=r+2$
 and $\mu$ or $\mu'$ lies in $\partial_{r+2}(\sigma)$,}\\
\partial_{q}\left(  \sigma\right)  , & \mbox{otherwise.}
\end{array}
\right.
\]
Then $\mu$ and $\mu'$ have \textbf{merged into} $\mu''$ \textbf{along} $\gamma$.
\end{definition}

\begin{proposition}\label{proposition_merged}
If regular cell complexes $X$ and $X^{\prime}$ are related as in Definition \ref{merging1}, then $H_{\ast}(X)\approx H_{\ast}(X^{\prime})$.
\end{proposition}

\begin{proof}
Consider a pair of  cellular maps $f:X\to X^{\prime}$ and $g:X^{\prime}\to X$ with the properties
\begin{itemize}
\item $f(\gamma)=f(\mu)=\partial \mu\smallsetminus \operatorname*{int}(\gamma)$, $f(\mu')=\mu''$, and $f(\sigma) =\sigma$, if $\sigma\in X$ and $\sigma \neq \gamma,\mu,\mu'$;
\item  $g(\mu'')=\mu\cup \mu'$ and  $g(\sigma) = \sigma$, 
if $\sigma\in X'$ and $\sigma \neq \mu''$.
\end{itemize}
Then $f$ and $g$ induce chain maps $f: C_{\ast}(X)\to C_{\ast}(X')$ and $g: C_{\ast}(X')\to C_{\ast}(X)$ such that
\begin{enumerate}
\item[$\bullet$] $f(\gamma)=\gamma+\partial(\mu)$, $f(\mu)=0$, $f(\mu')=\mu''$, and $f(\sigma)= \sigma$, if $\sigma\in\mathcal{X}$ and $\sigma \neq \gamma,\mu,\mu'$;
\item[$\bullet$] $g(\mu'')=\mu+\mu'$
 and $ g(\sigma) = \sigma$, 
if $\sigma\in \mathcal{X'}$ and $\sigma \neq \mu''$.
\end{enumerate}
There is a chain contraction $\left(f,g,\phi, (X,\partial), (X',\partial')\right)$, with $\phi:C_{\ast}(X)\to C_{\ast+1}(X)$ given by
$\phi(\gamma)=\mu$ and zero otherwise,
which is, in fact, a reduction of chain complexes in the sense of \cite{KMS}. 
The conclusion follows by Remark \ref{property-contraction}.
\end{proof}

\begin{figure}[t!]
\centerline{\hspace{4cm}\includegraphics[width=8cm]{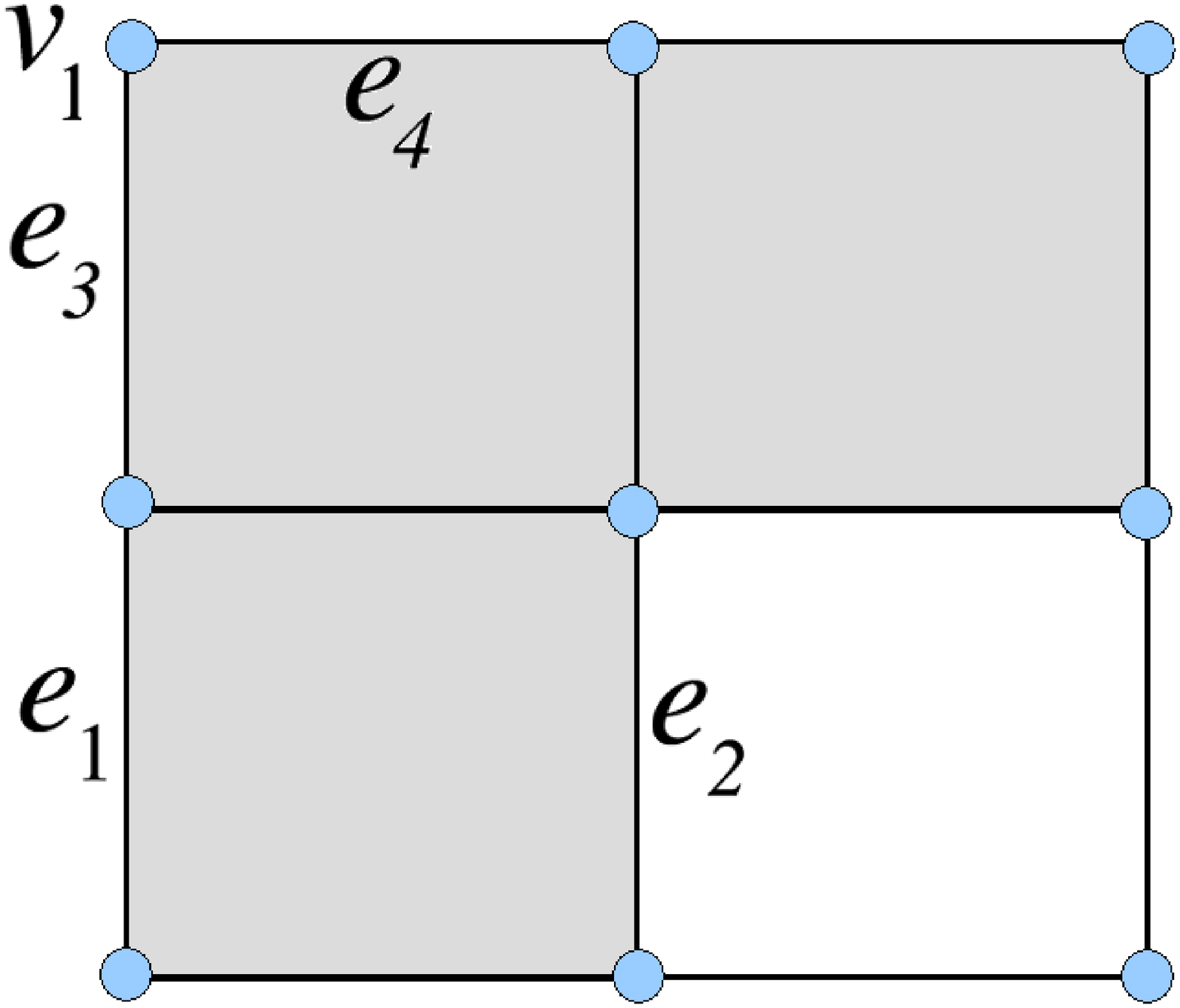}}
\vspace{-4cm}
$$\begin{array}{|l|c|} 
\hline
0-\mbox{cochain } c^0\; & \; A_{c^0}= \{v_1\}\\
1-\mbox{cochain }c^1\; & \;  A_{c^1}=\{e_1,e_4\}\\
1-\mbox{coboundary }\delta (c^0) \; & \; A_{\delta (c^0)} =\{e_3,e_4\}\\
1-\mbox{cocycle }x^1
\; & \; A_{x^1}
= \{e_1,e_2\}\\
1-\mbox{cocycle }y^1
\; & \; A_{y^1}
=\{e_1,e_2,e_3,e_4\}\\
\mbox{Cohomologous cocycles}\; & \; \mbox{$x^1$ and $y^1$; since $y^1+x^1$}
=\delta (c^0)\\\hline
\end{array}$$
 \caption{A cubical complex with $9$ vertices, $12$ edges and $3$ squares (in gray).}
\label{img:excohom}
\end{figure}

The dual space of $C_q(S)$ is denoted by $C^q(S)$ and its elements are called {\em $q$-cochains}.
Thus a $q$-cochain $c$ assigns a scalar value to each generator in $S_q=\{\sigma_1,\dots,\sigma_{n_q}\}$.  Let 
$A =\{\sigma \in S_q\mid c(\sigma)=1\}$; 
then $c$ is the characteristic function $\mathbb{I}_A: S_q \to \mathbb{Z}_2$. 
Let $a\in C_q(S)$ be a $q$-chain. Then $a=\sigma_{i_1}+\dots + \sigma_{i_r}$ for some indices $1 \leq i_1 < \dots < i_r \leq n_q$. 
The {\em cochain dual to} $a$ is the linear map $a^*:C_q(S)\to \mathbb{Z}_2$ such that $a^*(\sigma_i)=1$ if and only if  $i\in\{i_1,\dots,i_r\}$. 
The vector space of \emph{cochains of}
$S$ is the graded 
vector space
$C^*(S)=\{C^q (S)\}$.
The \emph{coboundary operator} is the linear map
$\delta=\{\delta^q: C^q(S)\to C^{q+1}(S)\}$ defined on a $q$-cochain $c$ by  $\delta^q(c)=c\partial_{q+1}
$; then $\delta\delta=0$
and the pair $(C^*(S),\delta)$ is the  \emph{cochain complex dual to} $(C_*(S),\partial)$.
A $q$-cochain $c$ is a $q$-\emph{cocycle} if $\delta^q (c)=0$. A $q$-cochain $b$ is a $q$-\emph{coboundary}
if there exists a $(q-1)$-cochain $c$ such that $b=\delta^{q-1}(c)$.
Two $q$-cocycles $c$ and $c'$ are \emph{cohomologous} if $c+c'$ is a $q$-coboundary. The $q$-cocycles $Z^q(S)$ and $q$-coboundaries $B^q(S)$ satisfy $B^q(S)\subseteq Z^{q}(S)$ since $\delta\delta=0$ (for concrete examples of cocyles and coboundaries, see the table in Figure \ref{img:excohom}).
The quotient $H^q(S)=Z^q(S)/B^q(S)$
is the $q^{th}$  \emph{cohomology
of} $S$ and the graded 
vector space
$H^*(S) = \{H^q(S)\}$
is the \emph{cohomology} of $S$.  An element of $H^q(S)$ is a class 
$[c] := c + B^q(S)$.

\begin{remark}\label{facts0}
Since the ground ring $\mathbb{Z}_2$ is a field,
the homology and cohomology
of a finite graded set $S$
are isomorphic and torsion free (see \cite[pages 325, 332--333]{Mun84}).
Furthermore, given
a chain contraction $\left( f,g,\phi,(S,\partial
),(S^{\prime},\partial^{\prime})\right) $, there is a map $\partial_{-}f:C_{\ast}\left(  S^{\prime}\right)  \rightarrow
C^{\ast}\left(  S\right)  $ that sends 
an element $\sigma\in S'$ to the cochain
$\partial_{\sigma}f$ defined by
\[
\partial_{\sigma}f\left(  \mu\right)  =\left\{
\begin{array}
[c]{cl}%
1, & \text{if }\sigma\text{ is a non-zero summand of }f\left(  \mu\right)  ;\\
0, & \text{otherwise;}%
\end{array}
\right.
\]
The cochain $\partial_{\sigma}f$ is supported on the
subspace generated by all $\mu\in S$
such that
$\sigma$ is a summand of $f\left(  \mu\right).$
\end{remark}

Let $S$ be a graded set, let $F \subseteq S$, and let
$0_F$
denote the 
zero 
differential on $C_*(F)$.
An \emph{Algebraic-Topological model} (AT-model) for $\left( C_*(S),\partial\right)$ is a chain contraction of the form
$\left(f,g,\phi,(S,\partial),(F,0_F)\right)$; since $0_F\equiv 0$, we often 
drop the symbol $0_F$ and
write $\left(f,g,\phi,(S,\partial),F\right)$.
An AT-model for $\left(  C_*(S),\partial\right)$ always exists with computational complexity no worse than ${\mathcal{O}}(m^{3})$, where $m=\#S$ (see \cite{GR03,GR05} for details). We note that the algorithm presented in \cite{KMS} for computing homology by reduction of chain complexes is a precursor of the AT-model. 

\begin{proposition}\label{facts1}
Let $\left(  f,g,\phi,(S,\partial),F\right)$ be an AT-model for $\left(  C_*(S),\partial\right)$.
\begin{enumerate}
\item[\textit{i.}] If $\sigma\in F_{q},$ then $g(\sigma) $ represents a class in
$H_{q}(S)  $.\smallskip
\item[\textit{ii.}] The map $F_{q}\rightarrow H_{q}(S)$ given by $\sigma
\mapsto\left[  g(\sigma)\right]  $ extends to an isomorphism $C_{q}(F)\approx
H_{q}\left(  S\right)  .$\smallskip
\item[\textit{iii.}] The map $F_{q}\rightarrow H^{q}\left(  S\right)  $ given
by $\sigma\mapsto\left[  \partial_{\sigma}f\right]  $ extends to an
isomorphism $\partial_{-}f:C_{q}(F)\approx H^{q}\left(  S\right).$
\end{enumerate}
\end{proposition}

\begin{proof}
In view of Remark \ref{facts0}, the proof is immediate.
\end{proof}

\begin{proposition}
\label{facts2}Let $(f,g,\phi,(S,\partial),(S^{\prime},\partial^{\prime}))$ be
a chain contraction.
\begin{enumerate}
\item[\textit{i.}] If $(f',g',\phi',(S,\partial),F)$ is an AT-model
for $\left(  C_{\ast}(S),\partial\right)  $, then $(f'g,fg',f\phi'g,$
$(S^{\prime},\partial^{\prime}),F)$ is an AT-model for $\left(  C_{\ast
}(S^{\prime}),\partial^{\prime}\right)  $.
\item[\textit{ii.}] If $(f'',g'',\phi'',(S^{\prime},\partial^{\prime
}),F^{\prime})$ is an AT-model for $\left(  C_{\ast}(S^{\prime}),\partial
^{\prime}\right)  $, then $(f''f,gg'',$ $\phi+g\phi''f,(S,\partial
),F^{\prime})$ is an AT-model for $\left(  C_{\ast}(S),\partial\right)  $.
\end{enumerate}
\end{proposition}

\begin{proof}
Observe that $f'\partial\equiv 0$,
$\partial g' \equiv 0$,
 $f''\partial'\equiv 0$ and $\partial' g''\equiv 0$. Hence 
\begin{itemize}
\item[i.]
$f'gfg'=f'(\partial\phi+\phi\partial+\mathbf{1}_{C_{\ast}\left(  S\right)
})g'=f'g'=\mathbf{1}_{C_{\ast}\left(  F\right)  }$ and
\[
\partial f\phi'g+f\phi'g\partial=f(\partial
\phi'+\phi'\partial)g=f(\mathbf{1}_{C_{\ast}\left(  S\right)  }+g'
f')g
=\mathbf{1}_{C_{\ast}\left(  S^{\prime}\right)
}+fg'f'g.\smallskip
\]
\item[ii.]  
 $f''fgg''=f''\left(\mathbf{1}_{C_{\ast}\left(S^{\prime}\right)  }\right)
g''=f''g''= \mathbf{1}_{C_{\ast}\left( F^{\prime}\right)  }$ and
\begin{align*}
\partial(\phi+g\phi''f)+(\phi+g\phi''f)\partial
& =\partial\phi+\phi\partial+g(\partial\phi''+\phi''\partial)f\\
& =\partial\phi+\phi\partial+g(\mathbf{1}_{C_{\ast}\left(  S^{\prime}\right)  }
+g''f'')f\\
& =\mathbf{1}_{C_{\ast}\left(  S\right)  }+gg''f''f.
\end{align*}
\end{itemize}
\end{proof}

\section{Cup Products and Diagonal Approximations}\label{section_cup}

In this section we review the definitions and related notions we need to compute the cup product on the cellular cohomology $H^{\ast}\left(  X\right)$ of a 
3D  polyhedral complex  $X$.

A {\em graded algebra} is a vector space $A$ equipped with an associative multiplication $m:A\otimes A\to A$ of degree zero and a unit. The cellular cohomology $H^*(X)$ of a regular cell complex $X$
is a {\em graded commutative algebra} with respect to the cup product 
defined below in this section. 
A {\em differential graded algebra} (DGA) is a chain complex $\left(A,d\right)$ in which $A$ is a graded algebra and $d$ a derivation of the product, i.e., $dm=m(d\otimes \mathbf{1} + \mathbf{1} \otimes d)$  (see \cite[page 190]{MacLane}).  Let $d^{\otimes}$ denote the linear extension of $d$ to the $n$-fold tensor product. A {\em homotopy associative algebra} is identical to a DGA to every respect except that the multiplication $m$ on $A$ is {\em homotopy associative}, i.e., there is a chain homotopy $m_3: A^{\otimes 3} \to A$ such that $m(m\otimes \mathbf{1}) + m(\mathbf{1}\otimes m)=d m_3 + m_3 d^{\otimes}$. Let $\tau:A\otimes A\to A\otimes A$ denote the twisting isomorphism
given by $\tau(a\otimes a')=a'\otimes a$. The multiplication $m$ is {\em homotopy commutative} if there is a chain homotopy $h:A\otimes A \to A$ such that $m + m\tau = dh + hd^{\otimes}$. Over a field, 
and $\mathbb{Z}_2$ in particular,
a {\em graded coalgebra} $C$ is the linear dual of a graded algebra, and comes equipped with a coassociative coproduct $\Delta:C \to C\otimes C$ of degree zero and a counit. A {\em differential graded coalgebra} (DGC) is the linear dual of a DGA. Thus a DGC is a chain complex $\left(C,d \right)$ in which $C$ is a graded coalgebra and $d$ is a coderivation of the coproduct, i.e., $\Delta d = (d\otimes \mathbf{1} + \mathbf{1}\otimes d)\Delta$. A homotopy coassociative (cocommutative) coproduct is the linear dual of a homotopy associative (commutative) product. 

The \emph{geometric diagonal} on a topological space $X$ is the map $\Delta^G_{\scst X}:X\rightarrow X\times X$ defined by 
$\Delta^{G}_{\scst X}(\sigma)=\sigma\times \sigma$. 
If $X$ is a regular cell complex, 
 $\Delta^{G}_{\scst X}$
is typically {\em  not} a cellular map, but nevertheless, the Cellular Approximation Theorem \cite[page 349]{at}) guarantees the existence of a cellular map $\Delta_{\scst X}$ homotopic to $\Delta^G_{\scst X}$. Roughly speaking, such maps are ``diagonal approximations'', the precise definition of which now follows (for a categorical treatment see \cite[page 250]{Spanier}).

\begin{definition}\label{diagaproximation}
Let $X$ be a regular cell complex. A \textbf{diagonal approximation on}
$X$
 is a
cellular map $\Delta_{\scst X}:X\rightarrow X\times X$ with the following properties:
\begin{enumerate}
\item[\textit{i.}] $\Delta_{\scst X}$ is homotopic to 
 $\Delta^{G}_{\scst X}$.
\medskip
\item[\text{ii.}] If $\sigma$ is a cell of $X$, then $\Delta_{\scst X}(\sigma)\subseteq
\sigma\times \sigma$.
 \medskip
\item[\textit{iii.}] $\Delta_{\scst X}$ extends to a chain map $\Delta_{\scst X} :C_{\ast}(X)\rightarrow C_{\ast}(X \times X)\approx C_{\ast}(X)\otimes C_{\ast}(X)$, called the \textbf{coproduct induced by} $\Delta_{\scst X}$. 
\end{enumerate}
\end{definition}

Item (iii) in Definition \ref{diagaproximation} implies that the cellular boundary map $\partial$ extends to a
coderivation of $\Delta_{\scst X}$. Thus $\left(  C_{\ast}(X),\partial,\Delta_{\scst X}\right)  $ is a DGC whenever $\Delta_{\scst X}$ is coassociative. Various examples of diagonal approximations appear in \cite{S-U}.

\begin{example}\label{AW}
Let $s_{n}=\langle 0,1,\dots, n\rangle  $ denote the $n$-simplex.
The  
Alexander-Whitney (A-W) diagonal approximation on $s_n$ extends to a chain map $\Delta_{s}: C_*(s_n) \to C_*(s_n)\otimes C_*(s_n)$ 
defined on the top-dimensional generator by
\[
\Delta_{s}(s_n) =\sum_{i=0}^{n}\langle 0,\dots, i\rangle  \otimes\langle  i,\dots, n\rangle
\]
(see \cite[page 250]{Spanier}).
The chain map $\Delta_{s}$ is a strictly coassociative, homotopy cocommutative coproduct on $C_{\ast}(s_{n})$.
Thus $\left(  C_{\ast}\left(  s_{n}\right)  ,\partial,\Delta_{s}\right)  $ is a DGC for each $n$.
\end{example}

The $0$-{\it cube} 
$\mathrm{I}^0$
is the single point $0$, the $1$-{\it cube} 
$\mathrm{I}^1$
is the unit interval $\mathrm{I}= \left[  0,1\right]  $, 
and the $n$-{\it cube} 
$\mathrm{I}^n$
is the $n$-fold Cartesian product 
$\mathrm{I}\times \cdots \times \mathrm{I}$ 
with $n\geq 1$ factors.
A cell of $\mathrm{I}^{n}$ is an $n$-fold Cartesian product
$u_{1} \times \cdots \times u_{n}\in\left\{  0,1,\mathrm{I}\right\}  ^{\times n}$.

\begin{example}\label{serre}
The %{\it 
Serre diagonal approximation 
on $ \mathrm{I}^n$ extends to a chain map $\Delta_{\mathrm{I}}:C_*( \mathrm{I}^n)\to C_*( \mathrm{I}^n)\otimes C_*( \mathrm{I}^n)$
defined on the top-dimensional generator by
\[
\Delta_{\mathrm{I}}\left(  0\right)  =0\otimes 0 \text{\; and \;}\Delta_{\mathrm{I}}(\mathrm{I}^{n})=\left( 0\otimes\mathrm{I}+\mathrm{I}\otimes1\right)^{\otimes n} 
\]
(see \cite[page 441]{Ser51}). Let $ac\otimes bd$ denote $(a\otimes b)\otimes (c\otimes d);$ then for example,
\[
\Delta_{\mathrm{I}}\left(  \mathrm{I}^{2}\right)  =  \left(  0\otimes
\mathrm{I}+\mathrm{I}\otimes1\right)  \otimes\left(  0\otimes\mathrm{I}%
+\mathrm{I}\otimes1\right)  =   0\,0\otimes\mathrm{I\,I}+0\,\mathrm{I}\otimes\mathrm{I}\,1+\mathrm{I}\,%
0\otimes1\,\mathrm{I}+\mathrm{I\,I}\otimes1\,1.
\]
Let $0'= \mathrm{I}$ and $\mathrm{I}' = 1$; then for $n \geq 1$ we obtain
\[
\Delta_{\mathrm{I}}\left(  \mathrm{I}^{n}\right)  =\sum_{u_{1}\cdots u_{n}%
\in\left\{  0,\mathrm{I}\right\}  ^{\otimes n}}u_{1}\cdots u_{n}\otimes
u_{1}^{\prime}\cdots u_{n}^{\prime}.
\]
Again, $\Delta_{\mathrm{I}}$ extends to a strictly coassociative, homotopy cocommutative coproduct
$\Delta_{\mathrm{I}}:C_{\ast}(\mathrm{I}^{n})\rightarrow C_{\ast}%
(\mathrm{I}^{n})\otimes C_{\ast}(\mathrm{I}^{n})$ and $\left(  C_{\ast}\left(
\mathrm{I}^{n}\right)  ,\partial,\Delta_{\mathrm{I}}\right)  $ is a DGC for each $n$.
\end{example}

Let $X$ be a regular cell complex, let 
 $\Delta_{\scst X}:X\rightarrow X\times X$
be a diagonal approximation, and let $(f,g,\phi,(X,\partial),F)$ be an
AT-model for $\left(C_*(X),\partial\right)$. Then 
$C_{*}(F) \approx  H_{*}(X)\approx H^{*}(X)$ by
Remark \ref{facts0} and Proposition \ref{facts1}.  Given classes 
$\alpha\in
H^{p}\left(  X\right)  $ and $\alpha'\in H^{q}\left(  X\right)  ,$ there exist unique
chains $a=\left(  \partial_{-}f\right)  ^{-1}\left(  \alpha\right)  \in
C_{p}\left(  F\right)  $ and\ $a'=\left(  \partial_{-}f\right)
^{-1}\left(  \alpha'\right)  \in C_{q}\left(  F\right)  $ such that $\alpha=[\partial
_{a}f]$ and $\alpha'=[\partial_{a'}f].$ The \emph{cup product} of
representative cocycles $\partial_{a}f$ and $\partial_{a'}f$ is the
$\left(  p+q\right)$-cocycle defined on $a''\in C_{p+q}\left(  X\right)
$ by%
\[
\left(  \partial_{a}f\smile\partial_{a'}f\right)  \left(
a''\right)  =m (\partial_{a}f\otimes\partial_{a'}f)\Delta_{\scst
X}(a''),
\]
where $m$ denotes multiplication in $\mathbb{Z}_2$. The \emph{cup product} of classes $\alpha=[\partial_{a}f]\in
H^{p}(X)$ and $\alpha'=[\partial_{a'}f]\in H^{q}(X)$ is the class
\[
\alpha\smile \alpha'=\left[  \partial_{a}f\smile\partial_{a'}f\right]  \in
H^{p+q}(X)
\]
(see \cite[page 251]{Spanier}).
In practice, we compute $\alpha\smile \alpha'$ in the following way:
 $\{g(\sigma):\; \sigma\in F_{p+q}\}$ is a complete set of representative $(p+q)$-cycles. 
For each $\sigma\in F_{p+q}$, let  $\lambda_{\sigma}=(\partial_{a}f\smile\partial_{a'}f)g(\sigma)$ and let $\sigma^*$ denote the cochain dual to $g(\sigma)$. Then 
\begin{equation}\label{productH}
\alpha\smile \alpha' = \sum_{\sigma\in F_{p+q}} \lambda_{\sigma} \left[\sigma^*\right].
\end{equation}

The cup product on $H^*(X)$ is a bilinear operation with unit. If $\Delta_{\scst X}$ is homotopy
coassociative, $\left(  H^{\ast}\left(  X\right)  ,\smile\right)  $ is a
graded algebra; if $\Delta_{\scst X}$ is also homotopy cocommutative, the algebra $\left(
H^{\ast}\left(  X\right)  ,\smile\right)  $ is graded commutative.  In particular, 
if $X$ is a simplicial (cubical) complex, the A-W diagonal $\Delta_s$ (Serre diagonal $\Delta_{\mathrm{I}}$)
induces an associative, graded commutative cup product on $H^{\ast}(X)$.  

In particular, if $X$ is a regular cell complex embedded in $\mathbb{R}^3$ then $H^{n}(X)=0$ for all $n>2$ (see \cite[chapter 10]{alexandroff}). Thus 
if $x,y\in H^{\ast}(X)$ are homogeneous elements of positive degree and $x\smile y \neq 0$, then $x,y\in H^1(X)$. 
Furthermore, since the squaring map $H^{1}(X)\rightarrow H^{2}(X)$ coincides with the Bockstein, a non-zero squaring map
implies that $\mathbb{Z}_{2}$ is a direct summand of both $H^{1}%
(X;\mathbb{Z})$ and $H^{2}(X;\mathbb{Z})$, which is impossible since $H^{2}(X;\mathbb{Z})=H_0(S^3\smallsetminus X,\mathbb{Z})$ by Alexander duality and $H_0(S^3\smallsetminus X,\mathbb{Z})$ is torsion free. Therefore all cup-squares vanish in $H^2(X)$, and we have proved: 

\begin{theorem}[Cohomology Structure Theorem] \label{wofsey}
Let $X$ be a regular cell complex embedded in $\mathbb{R}^3$. Then $H^{\ast}(X)$ is a direct sum of exterior algebras on 1-dimensional generators.
\end{theorem}

\remark {The argument that cup-squares vanish in $H^2(X)$ given above
for regular cell complexes $X$ embedded in $\mathbb{R}^3$, is due to E. Wofsey \cite{wofsey}.}\medskip

Let $X$ and $X^{\prime }$ be regular cell complexes, let $(f,g,\phi
,(X,\partial ),(X^{\prime },\partial ^{\prime }))$ be a chain contraction,
and let $(f^{\prime },g^{\prime },\phi ^{\prime },(X^{\prime },\partial
^{\prime }),F^{\prime })$ be an AT-model for $C_{\ast }(X^{\prime })$. Given 
$\sigma^{\prime }\in F_{p}^{\prime }$ and $\mu^{\prime }\in
F_{q}^{\prime }$, consider $\alpha^{\prime }=\left[ \partial _{\sigma^{\prime
}}f^{\prime }\right] $ and $\beta^{\prime }=\left[ \partial _{\mu^{\prime
}}f^{\prime }\right] $ in $H^{\ast }\left( X^{\prime }\right) ,$ and recall
that $(f^{\prime }f,gg^{\prime },\phi +g\phi ^{\prime }f,(X,\partial
),F^{\prime })$ is an AT-model for $C_{\ast }(X)$ (Proposition \ref{facts2}%
, part (ii)). Let $\alpha=\left[ \partial _{\sigma^{\prime }}f^{\prime }f\right] $
and $\beta=\left[ \partial _{\mu }f^{\prime }f\right] $ in $H^{\ast }\left(
X\right) $. Since $f:C_{\ast }(X)\rightarrow C_{\ast }\left( X^{\prime
}\right) $ is a chain homotopy equivalence, the map $f^{\ast }:H^{\ast }\left(
X^{\prime }\right) \rightarrow H^{\ast }\left( X\right) $ given by $f^{\ast }%
\left[ u\right] =\left[ u\circ f\right] $ is an algebra isomorphism. Thus 
$f^{\ast }\left( \alpha^{\prime }\right) =f^{\ast }\left[ \partial _{\sigma^{\prime }}f^{\prime }\right] 
=\left[ \partial _{\sigma^{\prime }}f^{\prime
}f\right] =\alpha,\text{ }f^{\ast }\left( \beta^{\prime }\right) =\beta,$ and 
$\alpha\smile \beta=f^{\ast }\left( \alpha^{\prime }\right) \smile
f^{\ast }\left( \beta^{\prime }\right) =f^{\ast }\left(
\alpha^{\prime }\smile \beta^{\prime }\right) $.
In summary we have proved:

\begin{proposition}
\label{nablaprima}If $X$ and $X^{\prime}$ are regular cell complexes, and $(f,g,\phi,(X,\partial
),$ $(X^{\prime},\partial^{\prime}))$ is a chain contraction, then 
for $\alpha^{\prime},\beta^{\prime} \in
H^{\ast}(X^{\prime})$ and $\alpha=f^*(\alpha'),\ \beta=f^*(\beta') $ we have
\[
f^{\ast }\left(
\alpha^{\prime }\smile \beta^{\prime }\right)= \alpha\smile \beta.
\]
\end{proposition}

\begin{definition}\label{coproduct}
Let $X$ and $X'$ be 
regular cell
complexes, let $\Delta_{\scst X}:C_{\ast}(X)\rightarrow C_{\ast}(X)\otimes
C_{\ast}(X)$ be the coproduct induced by a diagonal approximation, let $(f,g,\phi,$ $(X,\partial),(X^{\prime},\partial'))$ be a chain contraction, and 
let 
$(f',g',\phi',$ $(X',\partial'),F')$
 be an AT-model for $C_*(X')$. The {\em{\bf induced coproduct}} on $C_{\ast}(X')$ is the chain map
$$\widetilde{\Delta}_{\scst X'}:=(f\otimes f)\circ \Delta_{\scst X}\circ g:C_*(X')\to C_*(X')\otimes C_*(X')\label{induced}.$$
\end{definition}

\begin{proposition}\label{merge-diagonal} Let $(f,g,\phi,(X,\partial),(X',\partial'))$ be the explicit chain contraction in the proof of Proposition \ref{proposition_merged}. Then the induced coproduct $\widetilde{\Delta}_{\scriptscriptstyle X^{\prime}}$ defined in Definition \ref{coproduct} restricts to a diagonal approximation.
\end{proposition}

\begin{proof} 
We first check that $\widetilde{\Delta}_{\scriptscriptstyle X^{\prime}} (\mu'') \subseteq  \mu'' \times \mu''$.
Identify the generators $\sigma \in {\mathcal X}$ and $\sigma^{\prime} \in {\mathcal X'}$
with the corresponding cells $\sigma \in X$ and $\sigma^{\prime} \in X^{\prime}$, and
identify $\left(f\otimes f\right)(\sigma\otimes \sigma)$ with $\left(f \times f\right)(\sigma \times \sigma )$. 
Then 
\begin{align*}
\widetilde{\Delta}_{\scriptscriptstyle X^{\prime}} (\mu'')
& =\left[(f\times f)\circ\Delta_{\scriptscriptstyle X}\circ g\right](\mu'')
=(f\times f)\left[\Delta_{\scriptscriptstyle X} (\mu+\mu')\right]\\
& =(f\times f)(\Delta_{\scriptscriptstyle X}(\mu) +\Delta_{\scriptscriptstyle X}(\mu')).
\end{align*}
By assumption, $\Delta_{\scriptscriptstyle X} (\sigma) \subseteq \sigma\times \sigma$ for each cell $\sigma \subseteq X$; since $f$ vanishes on $\mu$ we have 
\[
(f\times f)(\Delta_{\scriptscriptstyle X}(\mu) +\Delta_{\scriptscriptstyle X}(\mu'))=(f\times f)\Delta_{\scriptscriptstyle X}(\mu')\subseteq \mu''\times \mu''.
\]
Next we show that $\widetilde{\Delta}_{\scriptscriptstyle X^{\prime}}$ is homotopic to the geometric diagonal map $\Delta_{\scriptscriptstyle X^{\prime}}^G$. Choose a homotopy $h:X\times I\rightarrow X\times X$ from $\Delta_{\scriptscriptstyle X}^{G}$
to $\Delta_{\scriptscriptstyle X};$ then at each $\sigma \in X$ we have $h\left(  \sigma,0\right)  =\Delta_{\scriptscriptstyle X}^{G}\left(  \sigma\right)  $
and $h\left(  \sigma,1\right)  =\Delta_{\scriptscriptstyle X}\left(  \sigma\right).$ 
Consider the composition 
\[
h^{\prime}=(f\times f)\circ h\circ (g \times \mathbf{1}) :X^{\prime}\times I \to X^{\prime}\times X^{\prime}.
\] 
Then at each $\sigma^{\prime}\in X^{\prime}$ we have
\begin{align*}
h^{\prime}\left(  \sigma^{\prime},0\right)   &  =\left[  \left(  f\times f\right)
\circ h\circ\left(  g\times\mathbf{1}\right)  \right]  \left(  \sigma^{\prime
},0\right)  =\left(  f\times f\right)  \left[  h\left(  g\left(  \sigma^{\prime
}\right)  ,0\right)  \right]   \\
&=\left(  f\times f\right)  \Delta_{X}^{G}\left(  g\left(  \sigma^{\prime
}\right)  \right)      =\left(  f\times f\right)  \left(  g\left(  \sigma^{\prime
}\right)  \times g\left(  \sigma^{\prime}\right)  \right)  \\
&=fg(\sigma^{\prime})\times
fg(\sigma^{\prime}) =\sigma^{\prime}\times \sigma^{\prime}=\Delta_{X^{\prime}}^{G}\left(  \sigma^{\prime
}\right)
\end{align*}
and%
\begin{align*}
h^{\prime}\left(  \sigma^{\prime},1\right)    & =\left[  \left(  f\times f\right)
\circ h\circ\left(  g\times\mathbf{1}\right)  \right]  \left(  \sigma^{\prime
},1\right)  =\left(  f\times f\right)  \left[  h\left(  g\left(  \sigma^{\prime
}\right)  ,1\right)  \right]  \\
& =\left[  \left(  f\times f\right)  \circ\Delta_{\scriptscriptstyle X}\circ g\right]  \left(
\sigma^{\prime}\right)  =\widetilde{\Delta}_{\scriptscriptstyle X^{\prime}}\left(  \sigma^{\prime}\right).
\end{align*}
It follows that $h^{\prime}$ is a homotopy from $\Delta_{\scriptscriptstyle X^{\prime}}^{G}$ to
$\widetilde{\Delta}_{\scriptscriptstyle X^{\prime}}.$
\end{proof}

\begin{proposition}\label{nablaprima1} Let $X$ be a regular cell complex, let $\Delta_{\scriptscriptstyle X}:C_{\ast}(X) \to C_{\ast}(X) \otimes C_{\ast}(X)$ be the coproduct induced by a diagonal approximation, and let $\widetilde{\Delta}_{\scriptscriptstyle X^{\prime}}$ be the induced coproduct defined in Definition \ref{coproduct}. 
\begin{itemize}
\item If $\Delta_{\scriptscriptstyle X}$ is homotopy cocommutative, so is $\widetilde{\Delta}_{\scriptscriptstyle X^{\prime}}$.
\item If $\Delta_{\scriptscriptstyle X}$ is homotopy coassociative, so is $\widetilde{\Delta}_{\scriptscriptstyle X^{\prime}}$.
\end{itemize}
\end{proposition}

\begin{proof}
We suppress subscripts and write $\Delta$ and $\widetilde{\Delta}$.
Assuming $\Delta  $ is homotopy cocommutative, choose a chain homotopy
$h:C_{\ast}\left(  X\right)  \rightarrow C_{\ast}\left(  X\right)  \otimes
C_{\ast}\left(  X\right)  $ such that $\partial^{\otimes}h+h\partial=\tau\Delta  +\Delta.$
Then 
\begin{align*}
\tau\widetilde{\Delta} +\widetilde{\Delta}   & =\tau\left(  f\otimes f\right)
\Delta  g+\left(  f\otimes f\right)
\Delta  g =\left(  f\otimes f\right)  \left(  \tau\Delta  %
+\Delta  \right)  g\\
& =\left(  f\otimes f\right)  \left(  \partial^{\otimes} h+h\partial
\right)  g =\partial^{\otimes} \left(  f\otimes f\right)  hg+\left(  f\otimes f\right)
hg\partial
\end{align*}
and  $\left(  f\otimes f\right)  hg$ is a chain homotopy from $\tau
\widetilde{\Delta} $ to $\widetilde{\Delta} $.  Assuming $\Delta  $ is homotopy coassociative,
choose a chain homotopy $h:C_*(X) \to C_*(X)^{\otimes 3}$ such that 
$\partial^{\otimes}h+h\partial =( \Delta   \otimes \mathbf{1}) \Delta   +( \mathbf{1}\otimes \Delta) \Delta  .$
Then $f^{\otimes 3}\left[  h+(  \Delta  \phi\otimes\mathbf{1}%
+\mathbf{1}\otimes\Delta  \phi)  \Delta  \right]  g$ is a chain
homotopy from $(  \tilde{\Delta} \otimes\mathbf{1})
\tilde{\Delta} $ to $(  \mathbf{1}\otimes\tilde{\Delta} )
\tilde{\Delta} $ as the following calculation demonstrates:\vspace{.1in}

\noindent $(  \tilde{\Delta} \otimes\mathbf{1}+\mathbf{1}%
\otimes\tilde{\Delta} )  \tilde{\Delta}  =\left[  \left(  f\otimes
f\right)  \Delta   g\otimes\mathbf{1}+\mathbf{1}\otimes\left(
f\otimes f\right)  \Delta   g\right]  \left(  f\otimes f\right)  \Delta   g \smallskip$

$=f^{\otimes 3}\left[  \Delta   gf\otimes\mathbf{1}+\mathbf{1}%
\otimes\Delta   gf\right]  \Delta   g\smallskip$

$=f^{\otimes 3}\left[  \Delta  \left(  \mathbf{1}+\partial\phi
+\phi\partial\right)  \otimes\mathbf{1}+\mathbf{1}%
\otimes\Delta  \left(  \mathbf{1}+\partial\phi+\phi\partial\right)
\right]  \Delta   g\smallskip$

$=f^{\otimes 3}\left[  \Delta  \otimes\mathbf{1}+\mathbf{1}%
\otimes\Delta  +\Delta  \left(  \partial\phi+\phi\partial\right)  \otimes
\mathbf{1}+\mathbf{1}\otimes\Delta  \left(  \partial\phi
+\phi\partial\right)  \right]  \Delta   g\smallskip$

$=f^{\otimes 3}[\Delta  \otimes\mathbf{1}+\mathbf{1}%
\otimes\Delta  +\left(  \Delta  \partial\phi\otimes\mathbf{1}+\Delta  
\phi\partial\otimes\mathbf{1}\right)
+\left(  \mathbf{1}\otimes
\Delta  \partial\phi+\mathbf{1}\otimes\Delta  \phi\partial\right)  ]\Delta  
g\smallskip$

$=f^{\otimes 3}[\Delta  \otimes\mathbf{1}+\mathbf{1}%
\otimes\Delta  +\partial^{\otimes}\left(  \Delta  \phi\otimes\mathbf{1}%
\right)  +\Delta  \phi\otimes\partial+\Delta  \phi\partial\otimes
\mathbf{1}+\partial^{\otimes}\left(  \mathbf{1}%
\otimes\Delta  \phi\right) \smallskip$

$\hspace{0.5in}\hspace{0.36in} +\partial\otimes\Delta  \phi+\mathbf{1}%
\otimes\Delta  \phi\partial]\Delta   g\smallskip$

$=f^{\otimes 3}[\left(  \Delta  \otimes\mathbf{1}+\mathbf{1}%
\otimes\Delta  \right)  \Delta  +\partial^{\otimes}\left(  \Delta  \phi
\otimes\mathbf{1}\right)  \Delta +\left(  \Delta  \phi\otimes\mathbf{1}%
\right)  \left(  \partial\otimes\mathbf{1}+\mathbf{1}%
\otimes\partial\right)  \Delta  \smallskip$

$\hspace{0.5in}\hspace{0.36in}+\partial^{\otimes}\left(  \mathbf{1}%
\otimes\Delta  \phi\right)  \Delta  +\left(  \mathbf{1}\otimes\Delta  
\phi\right)  \left(  \partial\otimes\mathbf{1}+\mathbf{1}%
\otimes\partial\right)  \Delta  ]g\smallskip$

$=f^{\otimes 3}[\partial^{\otimes}h+h\partial+\partial^{\otimes}\left(
\Delta  \phi\otimes\mathbf{1}\right)  \Delta  +\left(  \Delta  \phi
\otimes\mathbf{1}\right)  \Delta  \partial +\partial^{\otimes}\left(  \mathbf{1}%
\otimes\Delta  \phi\right)  \Delta \smallskip$

$\hspace{0.5in}\hspace{0.36in} +\left(  \mathbf{1}\otimes\Delta  
\phi\right)  \Delta  \partial]g\smallskip$

$=\partial^{\otimes}[f^{\otimes 3}\left[  h+\left(  \Delta  \phi
\otimes\mathbf{1}+\mathbf{1}\otimes\Delta  \phi\right)  \Delta  
g\right]+[f^{\otimes 3}\left[  h+\left(  \Delta  
\phi\otimes\mathbf{1}+\mathbf{1}\otimes\Delta  \phi\right)
\Delta   g\right]  \partial.$
\end{proof}

\section{Cup Products on 3D Polyhedral Approximations}\label{section_polyhedral}

Traditionally, one uses the standard formulas in \cite{Mun84,Ser51} to compute cup products on a simplicial or cubical complex.  Instead, we apply the diagonal approximation formula given by our next theorem to compute cup products on \emph{any} 
3D polyhedral complex
$X$. This is the main result in this paper.  

\begin{theorem}
\label{main}Let $X$ be a 3D polyhedral 
complex.
Arbitrarily number the vertices of $X$ from $1$ to $n$ and represent a
polygon $p$ of $X$ as an ordered $k$-tuple of vertices
$p=\langle i_{1},\ldots ,i_{k}\rangle$, where $i_1 = \min\{i_1, \dots, i_k\}$,
$i_1$ is adjacent to $i_k$, and $i_j$ is adjacent to $i_{j+1}$ for $1< j < k$. There is a diagonal approximation on $p$ given by%
\begin{align}
\widetilde{\Delta }_{\scriptscriptstyle p}(p)=& \langle i_{1}\rangle \otimes
p+p\otimes \langle i_{m\left( k\right) }\rangle +\sum_{j=2}^{m\left(
k\right) -1}(u_{2}+e_{2}+\cdots +e_{j-1}+\lambda _{j}e_{j})\otimes e_{j}
\notag \\
& +\sum_{j=m\left( k\right) }^{k-1}\left[ (1+\lambda
_{j})e_{j}+e_{j+1}+\cdots +e_{k-1}+u_{k}\right] \otimes e_{j},
\label{formula}
\end{align}
where $i_{m\left( k\right)}:=\max \left\{ i_{2},\ldots ,i_{k}\right\}$,
$\lambda _{j}=0$ if and only if $i_{j}<i_{j+1}$, 
$\left\{ u_{j}=\left\langle i_{1},i_{j}\right\rangle
\right\} _{2\leq j\leq k}$ and $\left\{ e_{j}=\left\langle
i_{j},i_{j+1}\right\rangle \right\} _{2\leq j\leq k-1}.$
\end{theorem}

\begin{proof}
Consider the
triangulation 
$\left\{ t_{j-1}=\right. $ $\left. \left\langle
i_{1},i_{j},i_{j+1}\right\rangle \right\} _{2\leq j\leq k-1}$ of $p$ and note that
the A-W diagonal on $t_{j-1}$ is given by%
\begin{eqnarray}
\Delta _{s}\left( t_{j-1}\right) &=&\lambda _{j}\left( \left\langle
i_{1}\right\rangle \otimes t_{j-1}+t_{j-1}\otimes \left\langle
i_{j}\right\rangle +u_{j+1}\otimes e_{j}\right)  \notag \\
&&\hspace*{0.3in}+\left( 1+\lambda _{j}\right) \left( \left\langle
i_{1}\right\rangle \otimes t_{j-1}+t_{j-1}\otimes \left\langle
i_{j+1}\right\rangle +u_{j}\otimes e_{j}\right) .
\label{A-W}
\end{eqnarray}
Our strategy is to merge these triangles inductively until $p$ is recovered
and the induced coproduct is obtained. We proceed by induction on $j.$

When $j=2$, set $p_{1}=t_{1}$
and note that either $i_{2}<i_{3},$ in which case $%
i_{m\left( 3\right) }=i_{3}$ and $\lambda _{2}=0,$ or $i_{2}>i_{3},$ in
which case $i_{m\left( 3\right) }=i_{2}$ and $\lambda _{2}=1.$ If $%
i_{m\left( 3\right) }=i_{3},$ Formula (\ref{formula}) gives the
non-primitive terms $\left[ u_{2}+\lambda _{2}e_{2}\right] \otimes e_{2}.$
Since $\lambda _{2}=0$ this expression reduces to $u_{2}\otimes
e_{2}=\lambda _{2}u_{3}\otimes e_{2}+\left( 1+\lambda _{2}\right)
u_{2}\otimes e_{2}.$ On the other hand, if $i_{m\left( 3\right) }=i_{2},$ Formula (\ref{formula}%
) gives the non-primitive terms $\left[ \left( 1+\lambda _{2}\right)
e_{2}+u_{3}\right] \otimes e_{2}.$ Since $\lambda _{2}=1$ this expression
reduces to $u_{3}\otimes e_{2}=\lambda _{2}u_{3}\otimes e_{2}+\left(
1+\lambda _{2}\right) u_{2}\otimes e_{2}.$ In either case, Formula (\ref%
{formula}) agrees with the A-W diagonal on $p_{1}$.

Now assume that for some $j\geq 3,$ Formula (\ref{formula}) holds on $%
p_{j-2}=\left\langle i_{1},\ldots ,i_{j}\right\rangle .$ 
Merge $p_{j-2}$ and
$t_{j-1}$ along $u_{j}$ and obtain $p_{j-1}=\left\langle i_{1},\ldots
,i_{j+1}\right\rangle ;$ we claim that Formula (\ref{formula}) also holds on
$p_{j-1}$. In the notation of Definition \ref{merging1}, set 
$\gamma=u_j$, $\mu= p_{j-2}$, $\mu' = t_{j-1}$, and $\mu''=p_{j-1}$.
Then
\[
\widetilde{\Delta }_{\scriptscriptstyle p_{j-1}}\left( p_{j-1}\right) 
 =\left( f\otimes f\right)(\Delta_{s}+\widetilde{\Delta}_{p_{j-2}})g\left( p_{j-1}\right) 
 =\left( f\otimes f\right)\{\Delta_{s}\left(t_{j-1}\right)+\widetilde{\Delta}_{p_{j-2}}\left( p_{j-2}\right) \}, 
\]
where $f$ and $g$ are the chain maps explicitly given in the proof of Proposition \ref{proposition_merged}.
Either $i_{j+1}>i_{m\left( j\right) },$ in which case $\lambda _{j}=0$ and $%
i_{m\left( j+1\right) }=i_{j+1},$ or $i_{m\left( j\right) }>i_{j+1},$ in
which case $i_{m\left( j+1\right) }=i_{m\left( j\right) }.$ First assume that $%
i_{j+1}>i_{m\left( j\right) }.$ Following the proof of Proposition \ref{proposition_merged},
define $f\left( u_{j}\right) =u_{2}+e_{2}+\cdots +e_{j-1},$ $f\left(
p_{j-2}\right) =0,$ and 
$f\left( t_{j-1}\right) =p_{j-1}.$
Then Formulas (\ref{formula}) and (\ref{A-W}) give
\begin{eqnarray*}
\widetilde{\Delta }_{\scriptscriptstyle p_{j-1}}( p_{j-1}) &=&\left( f\otimes f\right) \left\{
\left\langle i_{1}\right\rangle \otimes t_{j-1}+t_{j-1}\otimes \left\langle
i_{j+1}\right\rangle +u_{j}\otimes e_{j}\right.
\\
&&+\left\langle i_{1}\right\rangle \otimes p_{j-2}+p_{j-2}\otimes
\left\langle i_{m\left( j\right) }\right\rangle
 +\sum_{s=2}^{m\left(
j\right) -1}\left( u_{2}+e_{2}+\cdots +\lambda _{s}e_{s}\right) \otimes e_{s}
\\
&&+\sum_{s=m\left( j\right) }^{j-1}\left. \left[ \left( 1+\lambda
_{s}\right) e_{s}+e_{s+1}+\cdots +e_{j-1}+u_{j}\right] \otimes e_{s}\right\}
\\
&=&\left\langle i_{1}\right\rangle \otimes p_{j-1}+p_{j-1}\otimes
\left\langle i_{j+1}\right\rangle  \\
&&+\left( u_{2}+e_{2}+\cdots +e_{j-1}\right)
\otimes e_{j}+\sum_{s=2}^{m\left( j\right) -1}\left( u_{2}+e_{2}+\cdots
+\lambda _{s}e_{s}\right) \otimes e_{s} \\
&&+\sum_{s=m\left( j\right) }^{j-1}\left[ \left( 1+\lambda _{s}\right)
e_{s}+e_{s+1}+\cdots +e_{j-1}+\left( u_{2}+e_{2}+\cdots +e_{j-1}\right) %
\right] \otimes e_{s} \\
&=&\left\langle i_{1}\right\rangle \otimes p_{j-1}+p_{j-1}\otimes
\left\langle i_{j+1}\right\rangle +\sum_{s=2}^{j}\left( u_{2}+e_{2}+\cdots
+\lambda _{s}e_{s}\right) \otimes e_{s},
\end{eqnarray*}
which verifies Formula (\ref{formula}) in this case. On the other hand, if $%
i_{m\left( j\right) }>i_{j+1}\ $define $f\left( u_{j}\right) =e_{j}+u_{j+1},
$ $f\left( p_{j-2}\right) =p_{j-1},$ and 
$f\left( t_{j-1}\right) =0.$ Then
\begin{eqnarray*}
\widetilde{\Delta }_{\scriptscriptstyle p_{j-1}}\left( p_{j-1}\right) &=&\left( f\otimes f\right) \left\{ \lambda _{j}\left( \left\langle
i_{1}\right\rangle \otimes t_{j-1}+t_{j-1}\otimes \left\langle
i_{j}\right\rangle +u_{j+1}\otimes e_{j}\right) \right. \\ \\
&&+\left( 1+\lambda _{j}\right) \left( \left\langle
i_{1}\right\rangle \otimes t_{j-1}+t_{j-1}\otimes \left\langle
i_{j+1}\right\rangle +u_{j}\otimes e_{j}\right)  \\
&&+\left\langle i_{1}\right\rangle \otimes p_{j-2}+p_{j-2}\otimes
\left\langle i_{m\left( j\right) }\right\rangle
+\sum_{s=2}^{m\left(
j\right) -1}\left( u_{2}+e_{2}+\cdots +\lambda _{s}e_{s}\right) \otimes e_{s}
\\
&&+\sum_{s=m\left( j\right) }^{j-1}\left. \left[ \left( 1+\lambda
_{s}\right) e_{s}+e_{s+1}+\cdots +e_{j-1}+u_{j}\right] \otimes e_{s}\right\}
\\
&=&\left\langle i_{1}\right\rangle \otimes p_{j-1}+p_{j-1}\otimes
\left\langle i_{m\left( j\right) }\right\rangle
+\lambda _{j}u_{j+1}\otimes
e_{j}\\
&&+\left( 1+\lambda _{j}\right) \left( e_{j}+u_{j+1}\right) \otimes e_{j}
+\sum_{s=2}^{m\left( j\right) -1}\left( u_{2}+e_{2}+\cdots +\lambda
_{s}e_{s}\right) \otimes e_{s} \\
&&+\sum_{s=m\left( j\right) }^{j-1}\left[ \left( 1+\lambda
_{s}\right) e_{s}+e_{s+1}+\cdots +e_{j-1}+\left( e_{j}+u_{j+1}\right) \right]
\otimes e_{s} \\
&=&\left\langle i_{1}\right\rangle \otimes p_{j-1}+p_{j-1}\otimes
\left\langle i_{m\left( j\right) }\right\rangle +\sum_{s=2}^{m\left(
j\right) -1}\left( u_{2}+e_{2}+\cdots +\lambda _{s}e_{s}\right) \otimes e_{s}
\\
&&+\sum_{s=m\left( j\right) }^{j}\left[ \left( 1+\lambda _{s}\right)
e_{s}+e_{s+1}+\cdots +e_{j}+u_{j+1}\right] \otimes e_{s},
\end{eqnarray*}
which verifies Formula (\ref{formula}) in this case as well and completes
the proof.
\end{proof}

\section{Computing the $\mathbb{Z}_2$-cohomology Algebra of Polyhedral Approximations of 3D Digital Images}\label{nueva}

Given a 3D digital image $I=(\mathbb{Z}^3,26,6,B)$, we apply the simplification procedure presented below to obtain a polyhedral 
complex
$P(I)$. Next, we apply Theorem \ref{main} and adapt the algorithm for computing cup products given in \cite{GR03,GR05} to compute cup products in $H^*(P(I))$.

\subsection{3D Digital Images and Polyhedral 
Complexes}
\label{una}

Consider a 3D digital image $I=(\mathbb{Z}^3,26,6,B)$, where
$\mathbb{Z}^3$ is the underlying grid,
the {\em foreground} $B$ is a finite set of points in the grid,
and the {\em background} is $\mathbb{Z}^3\smallsetminus B$.
 We  fix the $26$-adjacency relation for the points of $B$ and the $6$-adjacency relation for the points of $\mathbb{Z}^3\smallsetminus B$. More concretely, two points $(x,y,z)$ and $(x',y',z')$ of $\mathbb{Z}^3$ are $26$-adjacent
if   $1\leq (x-x')^2+(y-y')^2+(z-z')^2\leq 3$; they are
$6$-adjacent if $(x-x')^2+(y-y')^2+(z-z')^2=1$.
The set of  unit cubes with faces parallel to the coordinate planes centered at the
points of $B$  (called the voxels of $I$) is the {\em continuous analog} of $I$ and is denoted by
$CI$.
The cubical complex $Q(I)$ associated to a 3D digital image $I$ is the set of
voxels (cubes) of $CI$ together with all of their faces (quadrangles, edges and vertices).
Observe that $(26, 6)$-adjacency implies that the topology of $CI$ reflects the topology of $I$
(i.e., the fundamental groups of $CI$ are naturally isomorphic to the digital fundamental
groups of the digital picture $I$ \cite{kong}).

The subcomplex  $\partial Q(I)$  consists of all cells of $Q(I)$ that are facets of exactly one (maximal) cell of $Q(I)$, and their faces.
Note that the maximal cells of $\partial Q(I)$ are all the quadrangles of $Q(I)$ shared by a voxel of $B$ and a voxel of $\mathbb{Z}^3\smallsetminus B$ (see Figure \ref{partialq}).

\begin{figure}[t!]
	\centering
		\includegraphics[width=10cm]{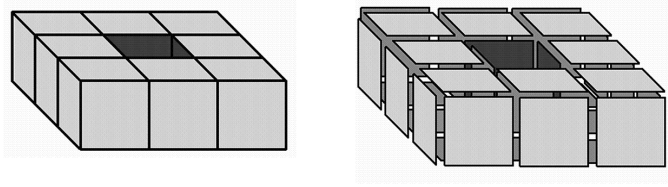}
	\caption{Left: A digital image $I=(\mathbb{Z}^3,26,6,B)$ 
(see Subsection \ref{una});
 the set 
$CI$
consists of $8$ unit cubes (voxels). Right: The quadrangles of $\partial Q(I)$.}
\label{partialq}
\end{figure}

\begin{figure}[t!]
	\centering
		\includegraphics[width=10cm]{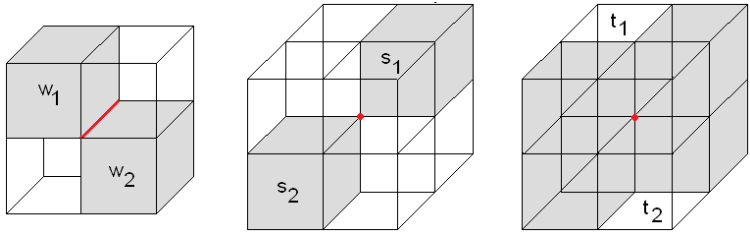}
	\caption{Critical configurations (i), (ii) and (iii) (modulo reflections and rotations).}
	\label{critical}
\end{figure}

We perform a simplification process in $\partial Q(I)$ to produce a 
3D polyhedral 
complex
 $P(I)$ homeomorphic to $\partial Q(I)$ whose maximal cells are polygons. But first, we need a definition.
\begin{definition}\label{defcritical}
A vertex $v\in\partial Q(I)$ is \textbf{critical} if any of the following conditions is satisfied:
\begin{enumerate}
\item[\textit{i.}] $v$ is a vertex of some edge $e$ shared by four cubes, exactly two of which lie in $Q(I)$ and intersect along $e$ (see cubes $w_1$ and $w_2$ in Figure \ref{critical}).\medskip
\item[\textit{ii.}] $v$ is shared by eight cubes, exactly two of which are corner-adjacent and contained in $Q(I)$ (see cubes $s_1$ and $s_2$ in Figure \ref{critical}).\medskip
\item[\textit{iii.}] $v$ is shared by eight cubes, exactly two of which are corner-adjacent and not contained in $Q(I)$ (cubes $t_1$ and $t_2$ in Figure \ref{critical}).\medskip
\end{enumerate}
A \textbf{non-critical vertex} of $\partial Q(I)$ lies in a neighborhood of $\partial Q(I)$ homeomorphic to $\mathbb{R}^2$ (see \cite{Lat97}).
\end{definition}

\begin{figure}[t!]
	\centering
		\includegraphics[width=11cm]{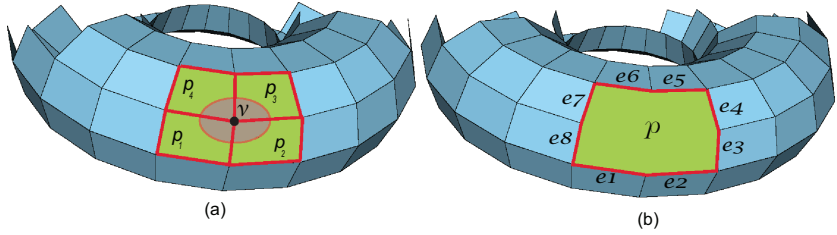}  
	\caption{(a) $N_v \leftarrow \{p_1,p_2,p_3,p_4\}; \;$
 (b) facets of $p$ $\leftarrow \{e_1,e_2,e_3,e_4,e_5,e_6,e_7,e_8\}$. 
See Definition \ref{defcritical}.}
	\label{toro}
\end{figure}

Let $V$ be the set of non-critical vertices in $\partial Q(I)$. Algorithm \ref{algo_disjdecomp} presented below processes the vertices in $V$
to obtain the 
3D polyhedral complex $P(I)$. Initially, $P(I)=\partial Q(I)$.
Given a vertex $v\in V$, let
$N_{v}$ denote the set of cells in $P(I)$ that are {\em incident} to $v$, i.e., $v$ is a vertex of each cell in $N_v$ ($v \in N_v$). 
We say that $v\in V$ is {\em removable} if 
 the number of $2$-cells in $N_v$ is greater than $2$ and, in this case,
 the cells of $N_v$ are replaced with the single $2$-cell $p$, which is the union of the cells in $N_{v}$ (see Figure \ref{toro}). 
Observe that the combinatorial 
criticality condition
in Definition \ref{defcritical} cannot be applied directly on $P(I)$.

Observe that the maximal cells of the resulting 3D polyhedral complex $P(I)$ are polygons and $ P(I)$ has fewer cells than $\partial Q(I)$.
We need some terminating conditions: 
\begin{itemize}
\item \textbf{Terminating Condition 1:} Terminate when 
for each removable non-critical vertex $v\in P(I)$, the number of edges in 
all polygons of $N_v$ is greater than or equal to some specified minimum $m$.\medskip
\item \textbf{Terminating Condition 2: }Terminate when 
for each removable non-critical vertex $v\in P(I)$, all $2$-cells of $N_{v}$ are coplanar.
\end{itemize}
The polygons of a polyhedral complex $P(I)$ produced using Terminating Condition 1 are (not necessarily planar) $k$-gons with $k>m$, whereas the polygons of $P(I)$ produced using Terminating Condition 2 are strictly planar.
Example \ref{polyhedral} demonstrates the differences that can arise from these terminating conditions.

\begin{alg}  
Obtaining the 3D polyhedral complex $ P(I)$.
\begin{tabbing}
{\sc Input:} {\tt  the cubical complex $\partial Q(I)$.
}\\
{\tt Initially, }\= {\tt $ P(I):=\partial Q(I)$;}\\
\>{\tt $V :=$ list of non-critical vertices of  $\partial Q(I)$.}\\
{\tt While} \= {\tt terminating condition is not satisfied}\\ 
\> {\tt For} \= {\tt $v\in V$: }\\
\> \> {\tt If} \= {\tt 
$v$ is removable
}\\  
\>\> \> {\tt remove the cells of $N_v$ from $ P(I)$;}\\
\>\>\> {\tt $2$-cell $p:=$ union of the cells in $N_{v}$;}\\
\> \> \> {\tt add  $p$  to $ P(I)$.}\\
\>\> {\tt End if;}\\
\>\> {\tt Remove $v$ from $V$.}\\
\> {\tt End for}\\
{\tt End while}\\
{\sc Output:}  {\tt The 3D polyhedral complex $ P(I)$.}
 \end{tabbing}
\label{algo_disjdecomp}
\end{alg}

\begin{proposition}
The homologies of $\partial Q(I)$  (the input of Algorithm \ref{algo_disjdecomp}) and $P(I)$ (the output of Algorithm \ref{algo_disjdecomp}) are isomorphic.
\end{proposition}

\begin{proof}
First, if $v\in P(I)$ is non-critical in $\partial Q(I)$, 
it is non-critical in $P(I)$.
Therefore
each
edge in $N_{v}$ is shared by exactly two $2$-cells. Recursively, take any edge $e$ in $N_v$ and the two $2$-cells
containing it, 
and merge these two cells into a new cell $p$
along $e$. 
This operation preserves homology
by Proposition \ref{proposition_merged}. 
When the process terminates we have 
the $2$-cell $p$
and the vertex $v$, which is an endpoint of some edge $e_v$ in the boundary of $p$ whose other endpoint is $w$. Now collapse
$e_v$ to $w$. Since this collapsing
operation
is a simple-homotopy equivalence, it preserves homology
(see \cite[pages 14--15]{simple}).
\end{proof}

\begin{figure}[t!]
	\centering
		\includegraphics[width=12.5cm]{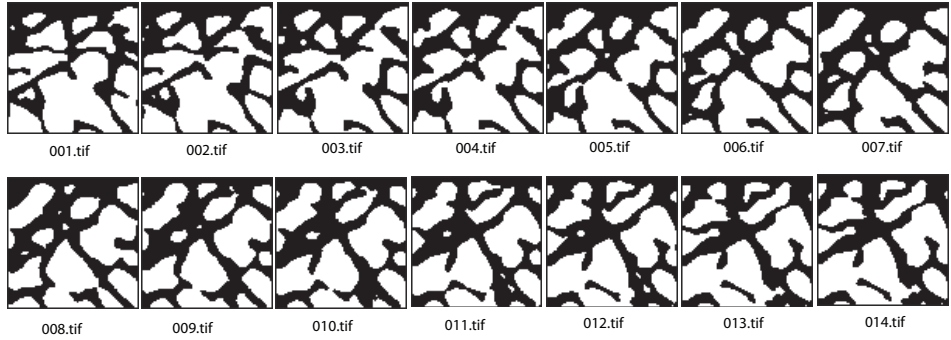}  
	\caption{A micro-CT of a trabecular bone.}
	\label{trabecularbone}
\end{figure}

\begin{figure}[t!]
	\centering
		\includegraphics[width=8cm]{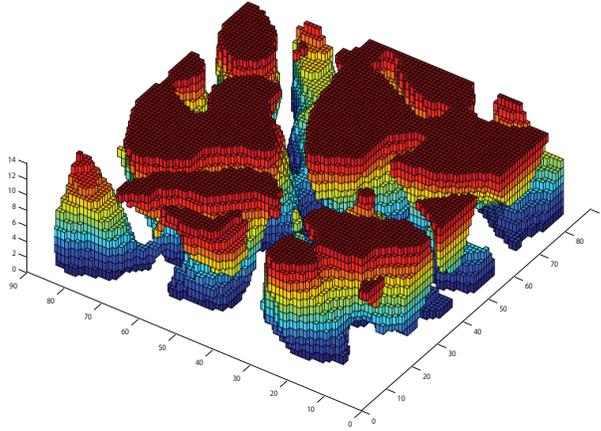}  
	\caption{The cubical complex 
$\partial Q(microCT)$.}
	\label{fig:cubical}
\end{figure}

Note that the size of the output depends on the choice of the vertex $v$.  Thus we conjecture that the procedure is optimized when $v$ is chosen to be the vertex of 
highest degree, i.e., the vertex with highest number of incident edges,
but we do not address this question here.  Nevertheless, our requirement that the number of $2$-cells in $N_v$ be at least $3$ ensures that the resulting complex is polyhedral and its $2$-cells are polygons.

Related algorithms for simplifying polygonal meshes appear in the
literature. We cite a few examples here; for a more extensive but
non-exhaustive list, see \cite{survey}. In their split-and-merge procedure,
F. Schmitt and X. Chen \cite{chen} use their \textquotedblleft merging
stage\textquotedblright\ procedure to join adjacent
\textquotedblleft nearly coplanar\textquotedblright regions. In \cite{lee}, J. Lee
shows how to simplify a triangular mesh by deleting vertices -- the result
is a new triangular mesh. In \cite{kalvin}, A. Kalvin, et al., show how to
reduce the complexity of a polygonal mesh by merging adjacent coplanar
rectangles. In \cite{gourdon}, A. Gourdon shows how to simplify a polyhedron
by sequentially removing edges while preserving the Euler characteristic.
And in \cite{klein}, R. Klein, et al., give a procedure for iteratively
removing vertices from a triangulated manifold, to produce a triangulated
polyhedron (all vertices in a manifold are non-critical).

\begin{figure}[t!]
	\centering
		\includegraphics[width=8cm]{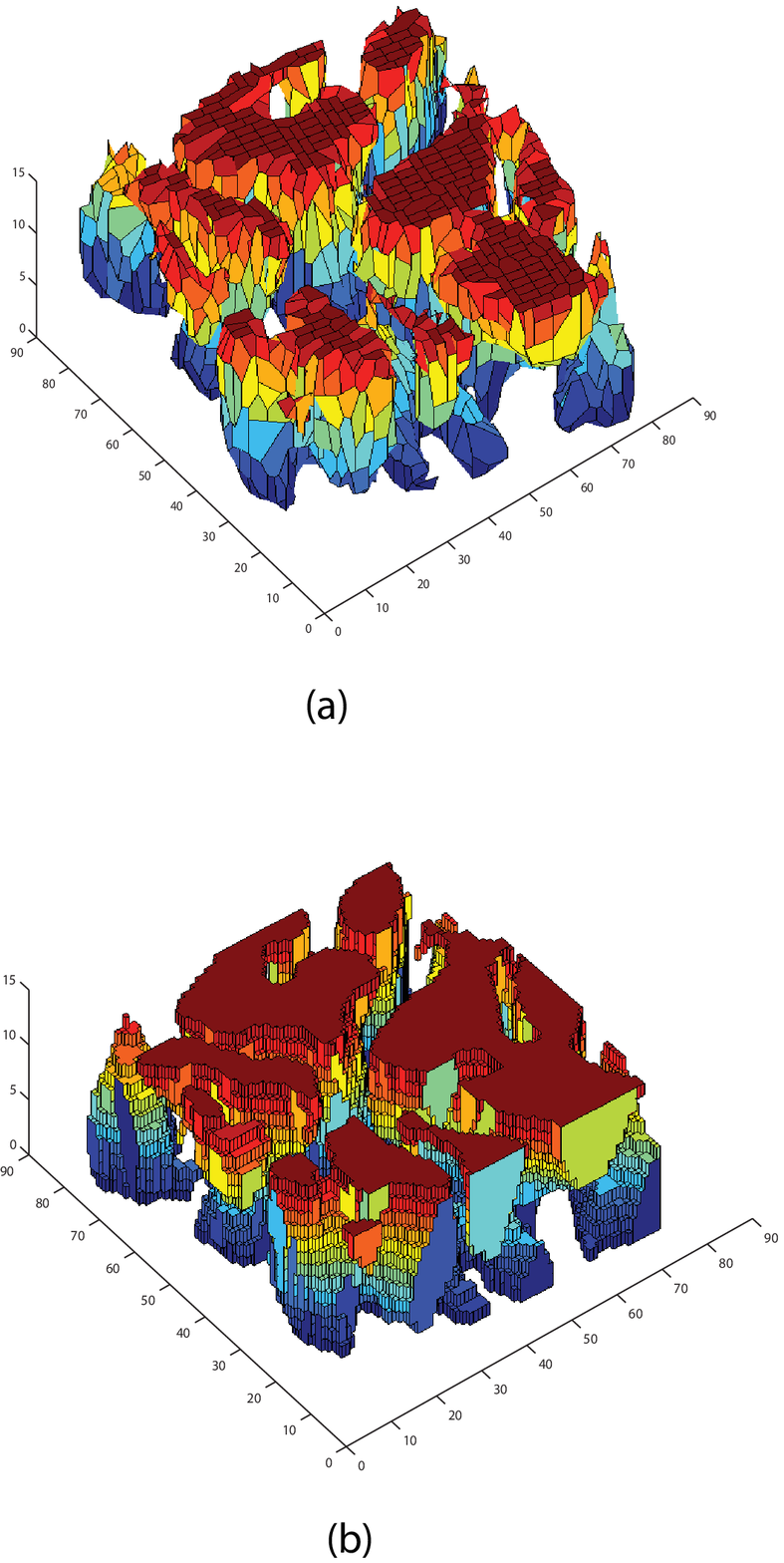}
	\caption{(a) The 
3D polyhedral 
complex 
$P(microCT)$
 for $10$ edges as 
lower
bound on
$p \in  P(microCT)$.
 (b) The 
3D polyhedral 
complex
$P(microCT)$} preserving  geometry.
	\label{reducuni}
\end{figure}

\begin{example}\label{polyhedral}
Figure~\ref{trabecularbone} displays a sequence of 14 2D digital images of size $85\times 85$ produced by a micro-CT of a trabecular bone.  Superposition produces a cubical complex we denote by $Q(microCT)$ with $85\times 85\times 14$ cubes. Its boundary $\partial Q(microCT)$ pictured in Figure \ref{fig:cubical} consists of $24582$ quadrangles.
An application of Algorithm \ref{algo_disjdecomp} to
$\partial Q(microCT)$ 
using Terminating Condition $1$ produces a polyhedral complex 
with $2158$ polygons (see Figure \ref{reducuni}.a).  On the other hand, if we apply Algorithm \ref{algo_disjdecomp} using Terminating Condition $2$, we obtain a polyhedral complex with $12802$ polygons (see Figure \ref{reducuni}.b).
\end{example}

 \begin{figure}[t!]
	\centering
		\includegraphics[width=10cm]{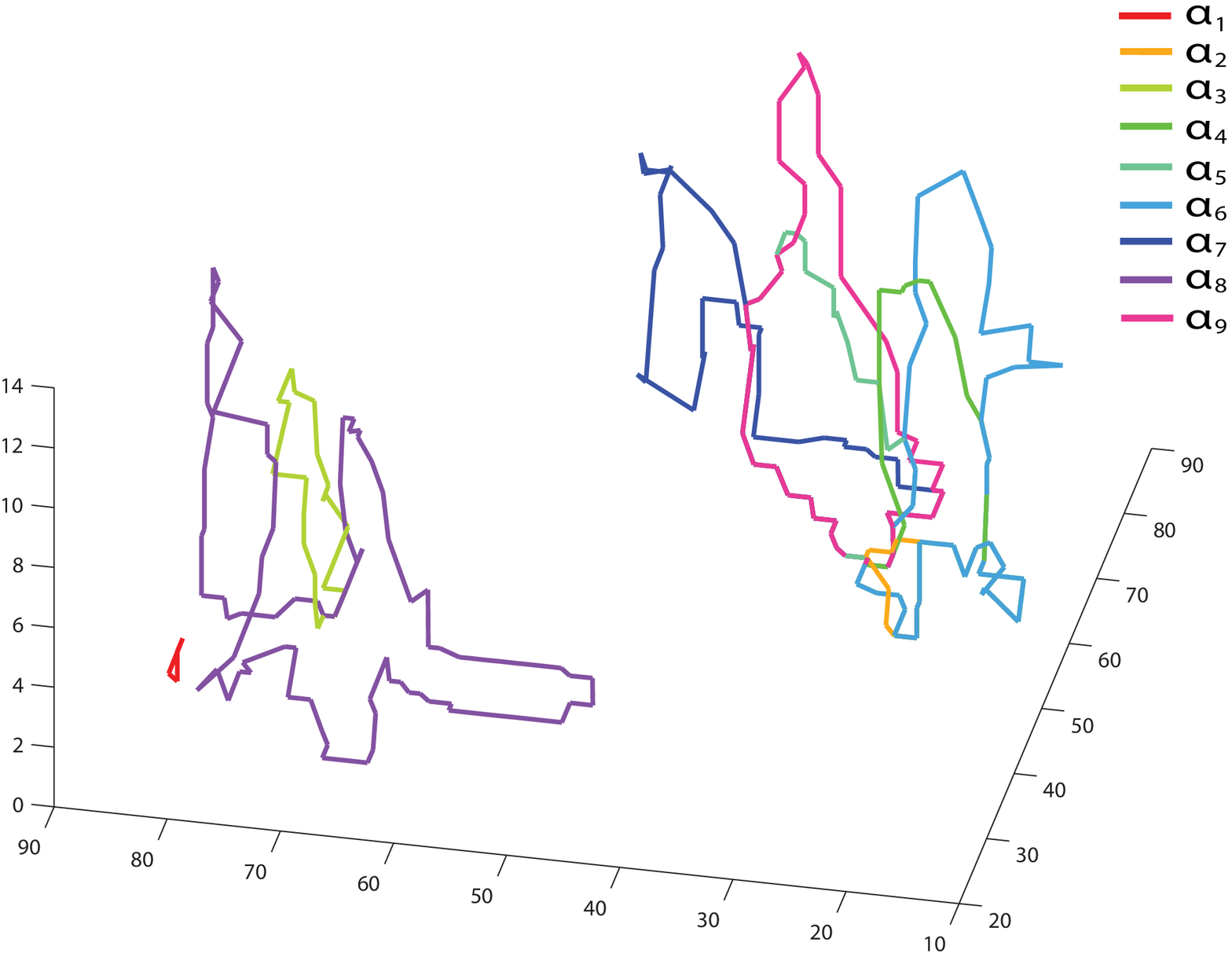}  
	\caption{Representative 
$1$-cycles on the 
3D polyhedral complex $P(microCT)$
displayed in Figure \ref{reducuni}.b.
}
	\label{cycles}
\end{figure}

Now, consider the output $P(I)$ of Algorithm \ref{algo_disjdecomp}. The chain complex\linebreak  $(C_*(P(I)),\partial)$ can be described as follows:
\begin{enumerate}
\item[$\bullet$] The vector space $C_q( P(I))$ is generated by the $q$-cells of $P(I)$.\medskip
\item[$\bullet$] The value of $\partial_q :C_q( P(I))\rightarrow C_{q-1}( P(I))$ on a $q$-cell is the sum of its facets.\medskip
\item[$\bullet$] The boundary of a sum of $q$-cells is the sum of their boundaries.
\end{enumerate}

\begin{example}\label{exatmodel}
Consider the 3D polyhedral complex $P(microCT)$ (Figure \ref{reducuni}.b)  
produced by an application of Algorithm \ref{algo_disjdecomp} to
$\partial Q(microCT)$
(Figure \ref{fig:cubical}).
The 
Betti numbers of $P(microCT)$
obtained by computing an AT-model for $P(microCT)$
are $b_0=8$, $b_1=9$ and $b_2=17$.  These results count the number of connected components, holes and cavities.   Representative $1$-cycles are pictured in Figure \ref{cycles}.
\end{example}

\subsection{Computing the Cohomology Algebra $H^*(P(I))$}\label{dos}

 Given a 3D digital image
$I$,
consider a  3D polyhedral 
complex
 $P(I)$ homeomorphic to $\partial Q(I)$.
All non-trivial cup products in $H^*(P(I))$ are products of distinct $1$-cocycles  by Theorem \ref{wofsey}.
To compute cup products in $H^*(P(I))$, we first compute an AT-model $(f,g,\phi,$ $(P(I),\partial),F)$ for $(C_*(P(I)),\partial)$
using Algorithm 2 given in \cite{GR05}.
If $F_1=\{\mu_0,\dots,\mu_{b_1-1}\}$ and $F_2=\{\gamma_0,\dots,\gamma_{b_2-1}\}$,
the cup products can be stored in a  $ b_1(b_1-1)/2 \times b_2$ matrix $A$. The entry $A(i+j,k)$, $0\leq i< j\leq b_1-1$, $0\leq k\leq b_2-1$, is:
 $$A(i+j,k)=m (\partial_{\mu_i} f\otimes \partial_{\mu_j} f)(\widetilde{\Delta}_{\scst P})g(\gamma_k),$$
where $\widetilde{\Delta}_{\scst P}$ is the induced diagonal approximation given by Theorem \ref{main} 
Thus by 
Equation 
\ref{productH} we have
$$[\partial_{\mu_i} f]\smile [\partial_{\mu_i} f]=\sum_{k=0}^{b_2-1}(A(i+j,k)) \cdot [\gamma^*_k]$$
where $\gamma^*_k$ is the cochain dual to $g(\gamma_k)$. 
The matrix $A$ is symmetric since the cup product is graded commutative ($\widetilde{\Delta}_{\scst P}$ is homotopy cocommutative by Proposition \ref{nablaprima1}).

Let $m$ be the number of voxels of $I$, let $n$ be the number of cells of   $ P(I)$, and let $k$ be the maximum number of vertices in a polygon of $P(I)$.
To determine the computational complexity of the computation of the
cohomology algebra $H^*(P(I))$, note that: 
\begin{enumerate}
\item[$\bullet$] Computing the number of critical vertices of $\partial Q(I)$ is $O(m)$ since the number of cells of $Q(I)$ is at most $27\cdot m$. \medskip

\item[$\bullet$] The computational complexity of Algorithm \ref{algo_disjdecomp}, 
which produces the 
3D polyhedral complex
$P(I)$,
 is $O(m)$ since, in the worst case, the number of edges in $P(I)$ incident to a non-critical vertex $v\in  P(I)$ is $6$ and the number of $2$-cells of $N_v$ is $12$.\medskip
\item[$\bullet$] The computational complexity of the algorithm given in \cite{GR05} to obtain an AT-model $(f,g,\phi,(P(I),\partial),F)$
for $(C_*(P(I)), \partial)$ is $O(n^3)$.  \medskip 
\item[$\bullet$] The complexity to compute a row of $A$ is at most  $O(n\cdot k^2\cdot b_1^2)$, since for a fixed $\gamma\in F_2$,
 $g(\gamma)$ has at most $n$ summands, $\widetilde{\Delta}_{\scst p}$ has $k$ summands, and for a $1$-cell $\sigma\in P(I)$,  $f(\sigma)$ has at most $b_1$ summands.
\end{enumerate}
Thus the overall computational complexity 
for computing the cohomology algebra $H^*(P(I))$
is $O(m+n^3+b_2\cdot n \cdot k \cdot b_1^2)$, where $b_i$ is the $i^{th}$ Betti number.
Furthermore, since    $k<<n<<m$ and $b_{\ell}<<n<<m$, $\ell=1,2$, overall complexity in most cases is $O(n^3)$ and at worst is $O(m^3)$ when no simplification of $\partial Q(I)$ is given.

\begin{figure}[t!]
	\centering
		\includegraphics[width=12cm]{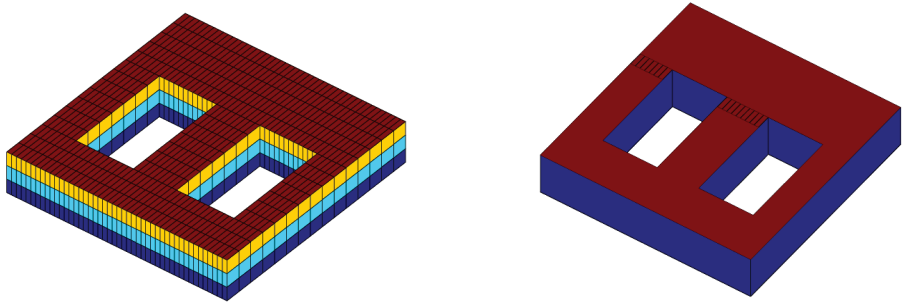} 
	\caption{Left: 
3D polyhedral complex $T$. Right: 
3D polyhedral complex $T'$.}
	\label{dobletoro}
\end{figure}

\begin{example}
Starting from the results obtained in Example \ref{exatmodel}, we applied the formula given in Theorem \ref{main} to
representative $1$-cocycles on the 3D polyhedral complex $P(microCT)$ (Figure \ref{reducuni}.b) and displayed their (non-trivial) cup products in 
the following table ($\alpha_i\alpha_j$ denotes $\alpha_i\smile\alpha_j$):
$$\begin{array}{c|c|c|c|c|c|}  
\qquad \quad \quad & \quad \alpha_{2} \alpha_{4} \quad &\quad  \alpha_{2}  \alpha_{5} \quad &\quad  \alpha_{2} \alpha_{9} \quad &\quad  \alpha_{3}  \alpha_{8}\quad &\quad \alpha_{4} \alpha_{5} \quad  \\
\hline
\beta_{16} & 1 & 1 & 1 & 1 & 1\\
\hline
\end{array}
$$
\end{example}

The table 
below illustrates the dramatic improvement in computational efficiency realized if
cup products are computed
 on the 3D polyhedral complex $T'$ obtained by removing faces and non-critical vertices of the 3D polyhedral complex $T$ of a (hollow) double torus (see Figure~\ref{dobletoro}).

$$\begin{array}{c|c|c|} 
\mbox{ 3D Polyhedral } & \mbox{ Number of } & \mbox{ Time (in seconds) }\\
\mbox{ complex }& \mbox{ $2$-cells  }& \mbox{ to compute the cup product } \\
\hline
T & 1638  & \mbox{ $28.00$ sec. }\\
\hline
T' & 46 & \mbox{ $1.04$ sec. }\\
\hline
\end{array}
$$

\vspace{0.5cm}

The following table displays the cup products on the polyhedral 
complex $T'$  (see Figure \ref{dobletoro}. Right).

$$
\begin{array}{c|c|c|c|c|c|c|}
 \quad&\quad \alpha_1\alpha_2   \quad&\quad  \alpha_1\alpha_3  \quad&\quad   \alpha_1\alpha_4
  \quad&\quad  \alpha_2\alpha_3  \quad&\quad  \alpha_2\alpha_4  \quad&\quad  \alpha_3\alpha_4\quad\\
\hline
\beta \quad & 0 & 1 & 1 & 0 & 1 & 0 \\
\hline
\end{array}
$$

\vspace{0.5cm}

\begin{figure}[t!]
	\centering
		\includegraphics[width=10cm]{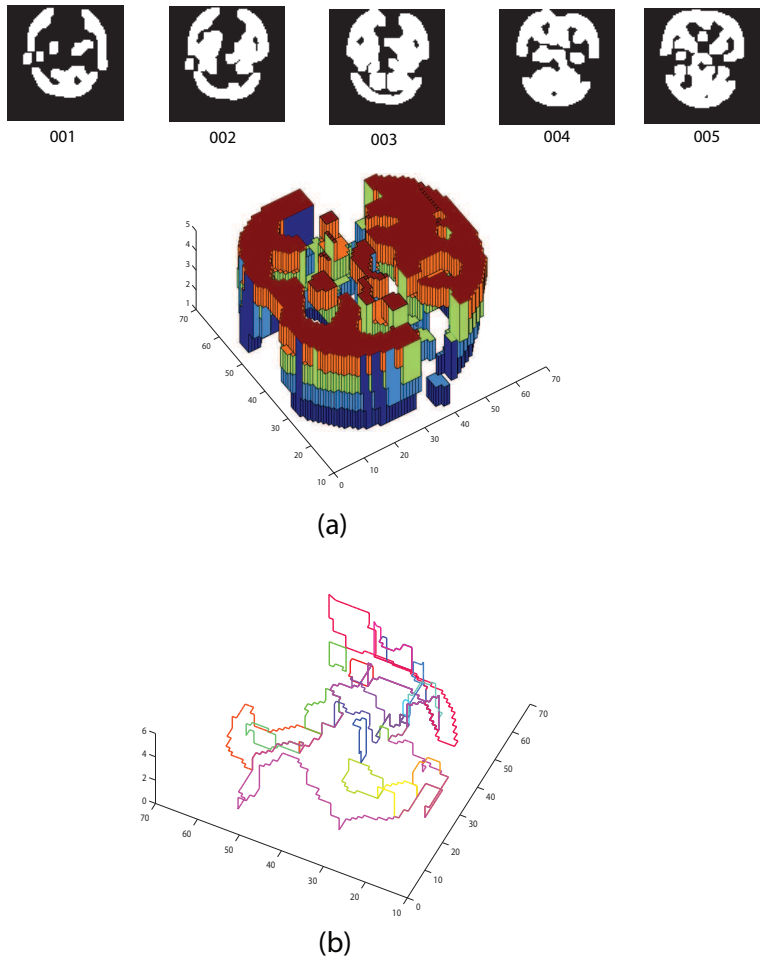}
	\caption{a)
3D polyhedral complex $P(DBrain)$; b) Representative $1$-cycles.}
	\label{figure14}
\end{figure}

\begin{example}
We obtained the 3D digital image $DBrain$ in Figure \ref{figure14}.a by binarizing and resizing the first five frames in J. Mather’s DICOM Example Files containing MR images of the brain\footnote{http://www.mathworks.com/matlabcentral/fileexchange/2762-dicom-example-files/}.
The boundary of $DBrain$ contains $7084$ quadrangles. Algorithm \ref{algo_disjdecomp}
with Terminating Condition 2 produces a 3D polyhedral complex $
P(DBrain)$ with $2433$ polygons (see Figure \ref{figure14}.a). The betti numbers $b_{0}=2$, $b_{1}=19$ and $b_{2}=2$
are determined by computing an AT-model 
$(f,g,\phi ,(P(DBrain),\partial ),F)$
 for $(C_{\ast}(P(DBrain)),\partial )$.
Denote  by $\{\alpha_1,\dots,\alpha_{19}\}$ and $\{\beta_1,\beta_2\}$, 
the respective bases for $H^1(DBrain)$ and $H^2(DBrain)$;
dual representative $1$-cycles
 are pictured in Figure \ref{figure14}.b. The non-trivial cup products displayed in the table below indicate the high level of topological complexity in the polyhedral complex $P(DBrain)$.
\begin{eqnarray*}
&&%
\begin{array}{c|c|c|c|c|c|c|c|c|c|}
& \alpha _{1}\alpha _{8} & \alpha _{2}\alpha _{3} & \alpha _{2}\alpha _{5} & 
\alpha _{2}\alpha _{17} & \alpha _{3}\alpha _{5} & \alpha _{3}\alpha _{6} & 
\alpha _{3}\alpha _{7} & \alpha _{3}\alpha _{9} & \alpha _{3}\alpha _{13} \\ 
\hline
\beta_{2} & 1 & 1 & 1 & 1 & 1 & 1 & 1 & 1 & 1 \\ \hline
\end{array}
\\
&& \\
&&%
\begin{array}{c|c|c|c|c|c|c|c|c|}
& \alpha _{5}\alpha _{6} & \alpha _{5}\alpha _{7} & \alpha _{5}\alpha _{9} & 
\alpha _{5}\alpha _{13} & \alpha _{6}\alpha _{12} & \alpha _{6}\alpha _{13}
& \alpha _{6}\alpha _{16} & \alpha _{8}\alpha _{19}   \\ \hline
\beta_{2} & 1 & 1 & 1 & 1 & 1 & 1 & 1 & 1   \\ \hline
\end{array}
\\
&& \\
&&%
\begin{array}{c|c|c|c|c|c|c|c|}
& \alpha _{9}\alpha _{17} & \alpha _{11}\alpha _{12} & \alpha _{12}\alpha
_{19} & \alpha _{13}\alpha _{15} & \alpha _{14}\alpha _{18} & \alpha
_{16}\alpha _{19} & \alpha _{18}\alpha _{19} \\ \hline
\beta_{2} & 1 & 1 & 1 & 1 & 1 & 1 & 1 \\ \hline
\end{array}%
\end{eqnarray*}
\end{example}

\section{Conclusions and Plans for Future Work}\label{tres}
 Given a 3D digital image $I$, we have formulated the cup product on the cohomology of the 
3D polyhedral complex $P(I)$ obtained by simplifying the cubical complex $\partial Q(I)$.  The algorithm presented here can be applied to any 3D polyhedral complex.
The ultimate goal of this work is to compute cup products on the cohomology of any regular $n$-dimensional cell complex over a general ring directly from its combinatorial structure (without subdivisions).  Our strategy will be to apply some standard topological constructions such as forming quotients, taking Cartesian products, and merging cells.
\vspace{.1in}

\noindent \textbf{Acknowledgments.} We wish to thank Jim Stasheff and the anonymous referees for their helpful suggestions, which significantly improved the exposition, and 
Manuel Eugenio Herrera Lara--Universidad Complutense de Madrid for providing the 14 micro-CT images in Figure \ref{trabecularbone}.

\end{document}